\documentclass[lettersize,journal]{IEEEtran}
\usepackage{amsmath,amsfonts}
\usepackage{algorithmic}
\usepackage{algorithm}
\usepackage{array}
\usepackage[caption=false,font=normalsize,labelfont=sf,textfont=sf]{subfig}
\usepackage{textcomp}
\usepackage{stfloats}
\usepackage{url}
\usepackage{verbatim}
\usepackage{graphicx}
\def\BibTeX{{\rm B\kern-.05em{\sc i\kern-.025em b}\kern-.08em
    T\kern-.1667em\lower.7ex\hbox{E}\kern-.125emX}}
\usepackage{balance}
\usepackage{amssymb,amsthm}
\usepackage{enumitem}
\usepackage[normalem]{ulem}

\newtheorem{assumption}{Assumption}
\newtheorem{theorem}{Theorem}

\newtheorem{lemma}{Lemma}

\usepackage{color}  
\usepackage{hyperref}
\hypersetup{
    colorlinks=true, 
    linktoc=all,     
    linkcolor=blue,  
}

\begin{document}

\title{HyperController: A Hyperparameter Controller for Fast and Stable Training of Reinforcement Learning Neural Networks\\
\thanks{}
}

\author{Jonathan Gornet,~\IEEEmembership{Student Member,~IEEE,}~Yiannis Kantaros,~\IEEEmembership{Member,~IEEE,}~Bruno Sinopoli,~\IEEEmembership{Fellow,~IEEE,} 
\thanks{J. Gornet, Y. Kantaros, and B. Sinopoli are with the Department of Electrical and Systems Engineering, Washington University in St. Louis, St. Louis, MO 63130, USA (email: jonathan.gornet@wustl.edu; bsinopoli@wustl.edu).}}

\maketitle

\begin{abstract}
    We introduce Hyperparameter Controller (HyperController), a computationally efficient algorithm for hyperparameter optimization during training of reinforcement learning neural networks. HyperController optimizes hyperparameters quickly while also maintaining improvement of the reinforcement learning neural network, resulting in faster training and deployment. It achieves this by modeling the hyperparameter optimization problem as an unknown Linear Gaussian Dynamical System, which is a system with a state that linearly changes. It then learns an efficient representation of the hyperparameter objective function using the Kalman filter, which is the optimal one-step predictor for a Linear Gaussian Dynamical System. To demonstrate the performance of HyperController, it is applied as a hyperparameter optimizer during training of reinforcement learning neural networks on a variety of OpenAI Gymnasium environments. In four out of the five Gymnasium environments, HyperController achieves highest median reward during evaluation compared to other algorithms. The results exhibit the potential of HyperController for efficient and stable training of reinforcement learning neural networks. 
\end{abstract}


\section{Introduction}

Hyperparameter optimization, a subarea of Automated Machine Learning \cite{hutter2019automated}, is critical for improving machine learning models. It focuses on finding the optimal configuration of hyperparameters, such as the learning rate or the batch size \cite{bergstra2011algorithms}. However, hyperparameter optimization cost functions may be unknown, expensive to compute, or nonsmooth (see details in \cite{garnett2023bayesian}), rendering many gradient-based optimization method such as stochastic gradient descent \cite{robbins1951stochastic} or ADAM \cite{kingma2014adam}, infeasible. Bergstra et al. \cite{bergstra2011algorithms} introduced grid and random search methods for hyperparameter optimization. This has provided a foundation for other hyperparameter optimization methods which have drawn inspirations from results in stochastic multi-armed bandits \cite{li2017hyperband,falkner2018bohb} (which study the problem of decision-making under uncertainty), genetic algorithms \cite{angeline1998evolutionary,back1998overview,6300857,clune2008natural}, and population-based approaches \cite{DBLP:journals/corr/abs-1711-09846}. Traditionally, hyperparameter optimization was implemented prior to the training process, using the same configuration throughout the whole training process \cite{bergstra2011algorithms,falkner2018bohb,yang2020hyperparameter}. Recently, methods have shifted to adjusting hyperparameters during the training process \cite{NEURIPS2019_743c41a9,li2018massively,li2019generalized,DBLP:journals/corr/abs-1711-09846,parker2020provably}.

This paper draws inspiration from Bayesian optimization due to its success such as its impact on AlphaGo's training \cite{chen2018bayesian} and other methods proposed in \cite{brochu2010tutorial,hennig2012entropy,nguyen2020knowing,NIPS2012_05311655}. Bayesian optimization methods focus on estimating the cost function while also searching for the optimal value. A key result in the area is the Gaussian Process Upper Confidence Bound (GP-UCB) algorithm proposed by \cite{srinivas2009gaussian}, which balances exploration and exploitation during optimization. Hybrid forms have appeared where elements from other optimization techniques have been integrated into Bayesian optimization such as random search \cite{ahmed2016we}, stochastic multi-armed bandit algorithms \cite{falkner2018bohb} (an extension of \cite{li2017hyperband}), and distributed optimization \cite{desautels2014parallelizing,gonzalez2016batch}.

This paper addresses the shortcomings of Bayesian optimization within the scope of Reinforcement Learning (RL). RL methods such as Trust Region Policy Optimization (TRPO) introduced by \cite{schulman2015trust}, later inspiring Proximal Policy Optimization (PPO) by \cite{schulman2017proximal}, have been proposed for monotonic improvement of the RL neural network with little hyperparameter tuning. Previously, a method called Population-Based Bandits (PB2) in \cite{parker2020provably}, which draws inspiration from Time-Varying Gaussian Process Upper Confidence Bound (TV-GP-UCB) in \cite{bogunovic2016time}, have proposed to adjust hyperparameters online during neural network training in conjunction with PPO to accomplish better monotonic improvement. However, the authors in \cite{parker2020provably} assume that the cost function evolves as a one-step autoregressive process, which may be too rigid for accurate prediction. In addition, both PB2 and other Bayesian optimization methods struggle with the curse of dimensionality \cite{garnett2023bayesian}. Finally, Bayesian optimization can become computationally expensive, requiring $\mathcal{O}\left(n^3\right)$ work to invert a $n \times n$ matrix where $n$ is the number of samples collected. Results in \cite{ambikasaran2015fast} have lowered the computational complexity to $\mathcal{O}\left(n \log^2 n\right)$ work, but further improvements are needed. 

To address these issues, this paper proposes the Hyperparameter Controller (HyperController), inspired by stochastic control theory. The key contribution of HyperController is its reduction in the required computational work by learning a more efficient representation of the hyperparameter's objective function to obtain efficient computation and finer control of the hyperparameters, ensuring that the RL neural networks obtains fast training and consistent improvement. 
This is contrasted with Bayesian optimization, which relies on the prior distribution of the hyperparameter's objective function. Therefore, HyperController obtains faster hyperparameter optimization. 

The contributions of this paper are summarized below. 
\begin{itemize}
    \item Inspired by prior work in Bayesian optimization, we model the hyperparameter optimization problem as an \textit{unknown} Linear Gaussian Dynamical System (LGDS), a mathematical model of a system with a state variable evolving linearly over time. The LGDS output is an objective function value given a hyperparameter configuration. 
    \item We introduce a method to discretize the LGDS state variable, allowing optimal one-step prediction of the LGDS output using the Kalman filter. 
    \item Given the unknown LGDS parameters, we propose learning a representation of the Kalman filter. This representation is computationally efficient to learn, requiring only $\mathcal{O}\left(s^3\right)$ work where $s$ is the size of the representation such that $s \ll n$. In addition, learning the representation allows finer control over the hyperparameters to obtain consistent improvement of the RL neural networks. 
    \item HyperController uses two strategies: optimizes each hyperparameter separately instead of searching for the global hyperparameter configuration and only selects the predicted optimal configuration according to the learned representation. This implementation is utilized to avoid searching a large hyperparameter space relative to the training iterations. 
    \item To quantify the performance of HyperController, we derive the upper bound on regret, which is the cumulative difference between the optimal objective function's response and its response based on the method's selected hyperparameter configuration. 
    \item We validate our theoretical results by using HyperController as a hyperparameter optimizer during training of RL neural networks within the OpenAI Gymnasium. 
\end{itemize}

The rest of the paper is structured as follows. Section \ref{sec:Problem Formulation} reviews the mathematical model of the hyperparameter optimization problem and how to approximate it. In Section \ref{sec:representation}, we discuss how to learn the model, leading to Section \ref{sec:algorithm} which is an introduction of the HyperController method. Section \ref{sec:Experiments} implements the HyperController method for training a RL neural network and compares it to a number of well-known results in the area. Section \ref{sec:Conclusion} concludes the paper and discusses future directions.

\smallskip
\noindent\textbf{Notation:} For any $x\in\mathbb{R}^n$ and $y\in\mathbb{R}^n$, we have the inner product $\left\langle x, y\right\rangle = x^\top y \in \mathbb{R}$. For the natural number $i \in \mathbb{N}$, we have $[i]\triangleq\{1,2,\dots,i\}$. %

\section{Problem Formulation}\label{sec:Problem Formulation}


For this section, we left the details of the problem formulation derivation in Appendices \ref{appendix:derivation}, \ref{appendix:discretize}, and \ref{appendix:Modeling} to help the flow of the paper. When considering hyperparameter optimization in machine learning, there exists two variables to optimize, $\theta(t) \in \mathcal{X} \subseteq \mathbb{R}^{m}$ and $\zeta(t) \in \mathcal{Y}\subseteq \mathbb{R}^{h}$ where $t \in [n] \triangleq \{1,2,\dots,n\}$ is the iteration number of the machine learning training process. The variable $\theta(t) \in \mathcal{X}$ is the machine learning model's parameters (i.e. neural network weights) while $\zeta(t) \in \mathcal{Y}$ would be the vector of hyperparameter values used in the training process (i.e. $\zeta(t) = \begin{pmatrix}
    \eta & N
\end{pmatrix}$ where $\eta$ is the stochastic gradient descent's learning rate and $N$ is the batch size). The cost functions for each variable are $L\left(\theta(t);\zeta(t)\right): \mathcal{X} \rightarrow \mathbb{R}$ and $\psi\left(\zeta(t);\theta(t)\right): \mathcal{Y} \rightarrow \mathbb{R}$. The goal is to solve the following optimization problems
\begin{align}
    \min_{\theta \in \mathcal{X}} L\left(\theta;\zeta(t)\right) \mbox{ given } \zeta(t) \in \mathcal{Y} \label{eq:ultimate_goal}  \\
    \max_{\zeta \in \mathcal{Y}} \psi\left(\zeta;\theta(t)\right) \mbox{ given } \theta(t) \in \mathcal{X} , \label{eq:optimization_problem} 
\end{align}
    
In machine learning, solving \eqref{eq:ultimate_goal} is core objective. However, since this problem is directly impacted by the hyperparameters $\zeta(t)$, how efficiently optimization problem \eqref{eq:ultimate_goal} is solved is impacted by optimization problem \eqref{eq:optimization_problem}. For this work, we will focus on developing strategies for solving \eqref{eq:optimization_problem}. Important details to note about \eqref{eq:optimization_problem} are as follows. First, the function $\psi\left(\zeta;\theta(t)\right)$ is harder to compute than $L\left(\theta(t);\zeta(t)\right)$. In addition, the function $\psi\left(\zeta;\theta(t)\right)$ is not analytic with respect to $\zeta$, i.e. the gradient does not exist. Finally, since $\theta(t) \in \mathcal{X}$ changes at each iteration $t$, the cost function $\psi\left(\zeta;\theta(t)\right)$ is time-varying. For addressing optimization problem \eqref{eq:optimization_problem}, we will model its dynamics.

We can approximate the changes in $\psi\left(\zeta;\theta\left(t\right)\right)$ with respect to $\theta\left(t\right)$, where $\psi\left(\zeta;\theta\left(t\right)\right)$ is an infinite dimensional vector in the $L^2\left(\mathbb{R}^h \rightarrow \mathbb{R}\right)$ space, with a finite dimensional vector $z_t \in \mathbb{R}^{d^h}$ ($d$ is the number of discretized values for each $\zeta[i]$, $i \in [h]$) defined below 
\begin{equation}
    \begin{pmatrix}
        z_t[1] \\
        z_t[2] \\
        \vdots \\
        z_t[d^{h}]
    \end{pmatrix} = \begin{pmatrix}
        \psi\left(\zeta_1[1],\dots,\zeta_1[h];\theta\left(t\right)\right) \\
        \psi\left(\zeta_2[1],\dots,\zeta_1[h];\theta\left(t\right)\right) \\
        \vdots \\
        \psi\left(\zeta_d[1],\dots,\zeta_d[h];\theta\left(t\right)\right) 
    \end{pmatrix} \label{eq:state_z_t_definition}, 
\end{equation}
where $\begin{pmatrix}
    \zeta_{A[1]}[1] & \dots & \zeta_{A[h]}[h]
\end{pmatrix} \in \mathcal{Y}$, $A \in [d]^h$. The finite dimensional vector $z_t$ is assumed to evolve linearly according to a Linear Gaussian Dynamical System (LGDS) shown below
\begin{equation}\label{eq:LGDS}
    z_{t+1} = \Gamma z_t + \xi_t, 
\end{equation}
where $z_t \in \mathbb{R}^{d^h}$ is the LGDS state variable and $\xi_t \in \mathbb{R}^{d^h}$ is process noise that models changes in $\psi\left(\zeta;\theta\left(t\right)\right)$ not reflected by the state matrix $\Gamma \in \mathbb{R}^{d^h\times d^h}$. It is assumed that the process noise $\xi_t$ is Gaussian distributed noise, i.e. $\xi_t \sim \mathcal{N}\left(\mathbf{0},Q\right)$ where $Q \succeq \mathbf{0}$. We impose the following assumptions on the LGDS defined as \eqref{eq:LGDS}:
\begin{assumption}\label{assum:stability}
    The state matrix $\Gamma$ is Schur or its eigenvalues of $\Gamma$ lie inside the unit circle, i.e. $\rho\left(\Gamma\right) < 1$. 
\end{assumption}
\begin{assumption}\label{assum:controllability}
    The matrix pair $\left(\Gamma,Q^{1/2}\right)$ is controllable, i.e. the following matrix has rank $d^h$ 
    \begin{equation}
        \begin{pmatrix}
            Q^{1/2} & \Gamma Q^{1/2} & \dots & \Gamma^{d^h-1}Q^{1/2} 
        \end{pmatrix}. 
    \end{equation}
\end{assumption}

Provided the state $z_t$ that approximates changes in $\psi\left(\zeta;\theta\left(t\right)\right)$, we model the response of the $\psi\left(\zeta;\theta\left(t\right)\right)$ with respect to a selected hyperparameter $\zeta \in \mathcal{Y}$ as 
\begin{equation}\label{eq:LGDS_output}
    X_t = \left\langle \mathbf{e}_{A}, z_t\right\rangle + \eta_t, 
\end{equation}
where the component of the vector $\mathbf{e}_{A}$ that is associated with the hyperparameter configuration $\psi\left(\zeta_{A[1]}[1],\dots,\zeta_{A[h]}[h];\theta\left(t\right)\right)$ in $z_t$ defined as \eqref{eq:state_z_t_definition} is 1 and the rest of the components are 0. The measurement noise $\eta_t \in \mathbb{R}$ is added to model errors of our approximation that is not reflected with time-correlated errors introduced by $\xi_t$. We will assume that the measurement noise $\eta_t$ is Gaussian distribution, i.e. $\eta_t \sim \mathcal{N}\left(0,\sigma^2\right)$. In addition, it is assumed that the state matrix $\Gamma$ and noise statistics $Q$ and $\sigma^2$ are \textit{unknown}. 

The motivation for using $z_t$ as an approximation for $\psi\left(\zeta;\theta\left(t\right)\right)$ is based on previous assumptions made in \cite{parker2020provably}, which uses theoretical results developed in \cite{bogunovic2016time}, and its simplicity yet effectiveness. In \cite{bogunovic2016time} and \cite{parker2020provably}, the authors assume that $\psi\left(\zeta;\theta\left(t\right)\right)$ is modeled by a function $f_t \in L^2\left(\mathbb{R}^h \rightarrow \mathbb{R}\right)$ that has the following dynamics
\begin{equation}
    f_{t+1} = \sqrt{1-\alpha} f_t + \sqrt{\alpha} \omega_t \nonumber, 
\end{equation}
where $\alpha \in [0,1]$ and $\omega_t \sim \mathcal{G}\mathcal{P}\left(\mathbf{0},\alpha W\left(\cdot,\cdot\right)\right)$. The distribution $\mathcal{G}\mathcal{P}\left(\mathbf{0},\alpha W\left(\cdot,\cdot\right)\right)$ is a Gaussian process that outputs random variables $\omega_t$ almost surely in the space $L^2\left(\mathbb{R}^h \rightarrow \mathbb{R}\right)$ if $W\left(\cdot,\cdot\right)$ is H\"{o}lder Continuous (see theorem in Section 2.5 of Garnett \cite{garnett2023bayesian}). The authors in both \cite{bogunovic2016time} and \cite{parker2020provably} assume that $\alpha$ and $W\left(\cdot,\cdot\right)$ are \textit{known}. Therefore, our formulation of $z_t \in \mathbb{R}^{d^h}$ differs from \cite{bogunovic2016time} and \cite{parker2020provably} assumptions by using a discretized approximation of $\psi\left(\zeta;\theta\left(t\right)\right)$ to model its changes as an \textit{unknown} LGDS \eqref{eq:LGDS}. In addition, $z_t$ avoid restrictions on how the function $\psi\left(\zeta;\theta\left(t\right)\right)$ is structured at a given time $t$; it only constrains $\psi\left(\zeta;\theta\left(t\right)\right)$ to evolve linearly. For example, the function $z_t$ can still be nonlinear with respect to $\zeta_1,\dots,\zeta_{d^h}$ even the changes in $\psi\left(\zeta;\theta\left(t\right)\right)$ are linear. 

The formulation above lead us to the following problem we aim to solve in this paper. Recall that the problem is to search for the hyperparameter configuration $\zeta_A \in \mathcal{Y}$ that maximizes $\psi\left(\zeta_A;\theta\left(t\right)\right)$ as efficiently as possible at each time $t$. Therefore, based on the discretization $z_t$ in \eqref{eq:LGDS} with output $X_t$ in \eqref{eq:LGDS_output}, we want to develop a method that finds the sequence of vectors $\mathbf{e}_{A_0}, \dots, \mathbf{e}_{A_n}$ (see \eqref{eq:LGDS_output}) that maximizes cumulative sum $\sum_{t=1}^n X_t \equiv \sum_{t=1}^n \psi\left(\zeta_{A_t};\theta\left(t\right)\right)$. To assess performance, we will analyze \textit{pseudo-regret} (referred to as regret for simplicity), which is the cumulative difference between the highest possible value of $\psi\left(\zeta^*\left(t\right);\theta\left(t\right)\right)$ and the value $\psi\left(\zeta_{A_t};\theta\left(t\right)\right)$:
\begin{equation}\label{eq:regret_def}
    R_n \triangleq\sum_{t=1}^n \psi\left(\zeta^*\left(t\right);\theta\left(t\right)\right)-\psi\left(\zeta_{A_t};\theta\left(t\right)\right). 
\end{equation}

\section{A Representation for Predicting the Output of the Hyperparameter Cost Function}\label{sec:representation}

Similar to what was done in Section \ref{sec:Problem Formulation}, we left the details on the model in Appendices \ref{appendix:representation}, \ref{appendix:alternating}, and \ref{appendix:identify}. To predict the cost function $\psi\left(\zeta;\theta\left(t\right)\right)$ for a given hyperparameter configuration, we will learn a representation that is efficient to compute. First, recall that the hyperparameter value $\zeta$ is a $h$-dimensional vector. If the hyperparameter value $\zeta$ is discretized as $\zeta_1,\dots,\zeta_{d^h}$, then a high value $d^h$ is required for accurately representing $\psi\left(\zeta;\theta\left(t\right)\right)$. However, it is possible that $d^h \gg n$, where $n$ is the total number of iterations in the machine learning training process. This leads to a situation where it is computationally intractable to explore all the responses $\psi\left(\zeta_A;\theta\left(t\right)\right)$ for each hyperparameter configuration $\zeta_A$, $A \in [d]^h$. To approach these issues, we will greedily search for each $\zeta_{A[i]}[i] = \zeta_a$, $i \in [h]$, such that $\zeta_A = \left(\dots,\zeta_a,\dots\right) \in \mathcal{D}$, $a \in [d]$, instead of search for the value $\zeta_A$ that globally optimizes the function. Therefore, we only have to search a space that is of size $d$ for each hyperparameter value $\zeta_a$ instead of the space of size $d^h$ for a global hyperparameter value $\zeta_A$. This proposed design provides the predictor that is learned based on the following optimization problem with its solution. 
\begin{multline}
    \underset{\tilde{G}_{a}\left(\mathbf{c}_i\right) \in \mathbb{R}^{s \times 1}}{\min} \left\Vert \mathbf{X}_{\mathcal{T}_{a \mid \mathbf{c}_i}} - \tilde{G}_{a}\left(\mathbf{c}_i\right)^\top \mathbf{Z}_{\mathcal{T}_{a \mid \mathbf{c}_i}}\right\Vert_2^2 \\ + \lambda \left\Vert \tilde{G}_{a}\left(\mathbf{c}_i\right)\right\Vert_2^2 \label{eq:identification_problem} , 
\end{multline}
\begin{align}
    \hat{G}_{a}^t\left(\mathbf{c}_i\right) & = \mathbf{X}_{\mathcal{T}_{a \mid \mathbf{c}_i}} \mathbf{Z}_{\mathcal{T}_{a \mid \mathbf{c}_i}}^\top V_a^t \left(\mathbf{c}_i\right)^{-1} \label{eq:learn_G} \\
    V_a^t \left(\mathbf{c}_i\right) & \triangleq \lambda I + \mathbf{Z}_{\mathcal{T}_{a \mid \mathbf{c}_i}}\mathbf{Z}_{\mathcal{T}_{a \mid \mathbf{c}_i}}^\top \label{eq:V_def},
\end{align}
where $\hat{G}_{a}^t\left(\mathbf{c}_i\right)$ is the minimizer of optimization problem \eqref{eq:identification_problem} and $\lambda > 0$. The sequence $\mathbf{c}_i = \left(\zeta_{a_{t-1}}, \dots, \zeta_{a_{t-s}}\right)$ is the previously selected hyperparameter values for $\zeta_{A[i]}[i]$, $i \in [h]$. The variables $\left(\mathbf{X}_{\mathcal{T}_{a \mid \mathbf{c}_i}},\mathbf{Z}_{\mathcal{T}_{a \mid \mathbf{c}_i}},\Xi_t\left(\mathbf{c}_i\right)\right)$ are defined to be 
\begin{align}
    \mathbf{X}_{\mathcal{T}_{a \mid \mathbf{c}_i}} & \triangleq \begin{pmatrix}
        X_{t_1} & \dots & X_{t_{N_a}}
    \end{pmatrix} \in \mathbb{R}^{1 \times N_a} \label{eq:X_def}\\
    \mathbf{Z}_{\mathcal{T}_{a \mid \mathbf{c}_i}} & \triangleq \begin{pmatrix}
        \Xi_{t_1}\left(\mathbf{c}_i\right) & \dots & \Xi_{t_{N_a}}\left(\mathbf{c}_i\right)
    \end{pmatrix} \in \mathbb{R}^{s \times N_a} \label{eq:Z_def} \\
    \Xi_t\left(\mathbf{c}_i\right) & \triangleq \begin{pmatrix}
        X_{t-s} & \dots & X_{t-1}
    \end{pmatrix}^\top \in \mathbb{R}^{s \times 1} \label{eq:xi_def}. 
\end{align}

Based on above, our prediction of $X_t$ in \eqref{eq:LGDS_output} for hyperparameter $\zeta_a$ using $\hat{G}_{a}^t\left(\mathbf{c}_i\right)$ is 
\begin{equation}
    \hat{G}_{a}^t\left(\mathbf{c}_i\right)^\top \Xi_t\left(\mathbf{c}_i\right) \label{eq:reward_prediction}. 
\end{equation}

The reward prediction \eqref{eq:reward_prediction} is a linear combination of previously observed rewards $X_{t-s},\dots,X_{t-1}$ where the linear parameters are from the learned vector $\hat{G}_{a}^t\left(\mathbf{c}_i\right)$. 

\section{Algorithm for Controller Hyperparameters During Neural Network Training}\label{sec:algorithm}

Using the learned $\hat{G}_{a}^t\left(\mathbf{c}_i\right)$, we propose the Hyperparameter Controller (HyperController)\footnote{Code implementation of HyperController is in \url{https://github.com/jongornet14/HyperController}} in Algorithm \ref{alg:HyperController}. HyperController selects hyperparameter values $\zeta_{A[i]}[i] = \zeta_a$, $\left(a,i\right) \in [d] \times [h]$ (we will use $a$ with relation to $\mathbf{c}_i$ for conciseness), based on the following optimization problem 
\begin{equation}\label{eq:action_selection}
    \zeta_{A[i]}[i] = \zeta_a \mbox{ s.t. } a = \underset{a \in [d]}{\arg\max} ~ \hat{G}_{a}^t \left(\mathbf{c}_i\right)^\top \Xi_t\left(\mathbf{c}_i\right) . 
\end{equation}

\begin{algorithm}[t]
\caption{Hyperparameter Controller (HyperController)}\label{alg:HyperController}
    \begin{algorithmic}[1]
        \STATE \verb|/* Input */|
        \STATE \textbf{Input: } $\lambda > 0$, $\mathcal{Y}$, $s,h,d \in \mathbb{N}$
        \STATE \verb|/* Initialization */|
        \STATE $t \gets 0$
        \FOR{$\zeta_{A[i]}[i]$, $i\in [h]$}
            \STATE $\mathcal{D}_i = \left\{\zeta_{A[i]}[i] = \zeta_a \mid \left\vert \zeta_a - \zeta_{a+1}\right\vert = \Delta_i,~a \in [d]\right\}$
            \FOR{$\mathbf{c}_i \in \left\{\begin{pmatrix}
                \zeta_{a_1}[i] & \dots & \zeta_{a_s}[i]
            \end{pmatrix} \in \mathcal{D}_i^s\right\}$}
                \FOR{$\left(a,i\right) \in [d] \times [h]$}
                    \STATE $\mathcal{T}_{a\mid\mathbf{c}_i}^t \gets \{\}$
                    \STATE $V_a^t\left(\mathbf{c}_i\right) \gets \lambda I_s$
                    \STATE $B_{a}^t \left(\mathbf{c}_i\right) \gets \mathbf{0}_{1 \times s}$ 
                    \STATE $\hat{G}_{a}^t \left(\mathbf{c}_i\right) \gets \mathbf{0}_{s \times 1}$ 
                \ENDFOR
            \ENDFOR
        \ENDFOR
        \STATE \verb|/* Interaction */|
        \FOR{$t \in [n]$}
            \FOR{$i \in [h]$}
                \IF{ $t \geq s$ }
                \STATE \verb|/* Action Selection */|
                \STATE $a_t \gets \underset{a \in [d]}{\arg\max} ~ \hat{G}_{a}^t \left(\mathbf{c}_i\right)^\top \Xi_t\left(\mathbf{c}_i\right)$
                \STATE \verb|/* Update Models */|
                \STATE $\mathbf{c}_i \gets\begin{pmatrix}
                    \zeta_{a_{t-s}}[i] & \dots & \zeta_{a_{t-1}}[i]
                \end{pmatrix}$ 
                \STATE $\mathcal{T}_{a_t\mid\mathbf{c}_i}^t \gets\mathcal{T}_{a_t\mid\mathbf{c}_i}^t \cup \{t\}$
                \STATE $V_{a_t}^t\left(\mathbf{c}_i\right) \gets V_{a_t}^t\left(\mathbf{c}_i\right) + \Xi_t\left(\mathbf{c}_i\right) \Xi_t\left(\mathbf{c}_i\right)^\top$ 
                \STATE $B_{a_t}^t \left(\mathbf{c}_i\right) \gets B_{a_t}^t \left(\mathbf{c}_i\right) + X_t \Xi_t\left(\mathbf{c}_i\right)^\top $ 
                \STATE $\hat{G}_{a_t}^t \left(\mathbf{c}_i\right) \gets B_{a_t}^t \left(\mathbf{c}_i\right) V_{a_t}^t\left(\mathbf{c}_i\right)^{-1}$ 
                \ELSE
                \STATE $a_t \gets \mbox{Sample uniformly } a \sim [d]$ 
                \ENDIF
                \STATE $\zeta_{A_t[i]}[i] \gets \zeta_{a_t}$
                \STATE Observe $X_t$
            \ENDFOR
        \ENDFOR
    \end{algorithmic}
\end{algorithm}

In \eqref{eq:action_selection}, we are selecting actions that maximize $\psi\left(\dots,\zeta_a,\dots;\theta\left(t\right)\right)$ with respect to one hyperparameter value $\zeta_{A[i]}[i]=\zeta_{a}$ while ignoring the selected hyperparameter values $\zeta_{A[j]}[j]=\zeta_{a'}'$, $j \neq i$. This provides a smaller search space---space is size $d$ instead of $d^h$---and requires lower memory usage allocated for each $\hat{G}_{a}^t\left(\mathbf{c}_i\right)$, $i \in [h]$. 

The overview of the algorithm is as follows. 

\verb|Input|: For the inputs of HyperController, there are the values $\lambda > 0$, $\mathcal{Y}$, and $s,h,d \in \mathbb{N}$. 
\begin{itemize}
    \item $\lambda > 0$: Recall that the matrix $V_a^t\left(\mathbf{c}_i\right)$ \eqref{eq:V_def} has a matrix $\lambda I$. This matrix $\lambda I$ ensures that the matrix $V_a^t\left(\mathbf{c}_i\right)$ is always invertible. 
    \item $\mathcal{Y}$: This is the set of possible hyperparameter values. For example, if there are the hyperparameters learning rate $\eta \in \left[10^{-5},10^{-1}\right]$ and batch size $N \in \left[32,128\right]$, we would input the set $\left\{\left(\eta,N\right) \in \left[10^{-5},10^{-1}\right] \times \left[32,128\right]\right\}$. 
    \item $s \in \mathbb{N}$: This value is how many rewards $X_{t-1},\dots,X_{t-s}$ in the past HyperController uses for predicting the next reward $X_t$. We found in our results that $s = 1$ works best, since the number of models $\hat{G}_{a}^t\left(\mathbf{c}_i\right)$ increases exponentially as $s$ grows. 
    \item $h \in \mathbb{N}$: This is the number of hyperparameters. For example, if we have hyperparameters $\eta$ and $N$ mentioned earlier, then $h = 2$. 
    \item $d \in \mathbb{N}$: The value $d$ is the number of discretized values to use for each hyperparameter. We use a uniform value $d$ since we have found that this simplifies the implementation. Based on our numerical results, we recommend using values within the set $[10,100]$. 
\end{itemize}

\verb|Initialization|: For initialization, we set $t = 0$, $\mathcal{D}_i$, $i \in [h]$, $\mathcal{T}_{a\mid\mathbf{c}_i}^t$, $V_a^t\left(\mathbf{c}_i\right)$, $B_{a}^t \left(\mathbf{c}_i\right)$, and $\hat{G}_{a}^t \left(\mathbf{c}_i\right)$
\begin{itemize}
    \item $\mathcal{D}_i$: For each hyperparameter $\zeta[i] \in \mathcal{Y}$, $i \in [h]$, we discretize the hyperparameter into $d$ values that are uniformly separated. We denote the distance between each discretized hyperparameter as $\Delta_i$. 
    \item $V_a^t\left(\mathbf{c}_i\right)$, $B_{a}^t \left(\mathbf{c}_i\right)$, and $\hat{G}_{a}^t \left(\mathbf{c}_i\right)$: The matrix $V_a^t\left(\mathbf{c}_i\right)$ and the vectors $B_{a}^t \left(\mathbf{c}_i\right)$ and $\hat{G}_{a}^t \left(\mathbf{c}_i\right)$ are initialized to values $\lambda I$, $\mathbf{0}_{1 \times s}$, and $\mathbf{0}_{s \times 1}$ respectively. These values are used for identifying the representation that is used for predicting each hyperparameter value's output. The advantage of using $V_a^t\left(\mathbf{c}_i\right) \in \mathbb{R}^{s \times s}$ is that inverting this matrix only requires $s^3$ work. 
\end{itemize}

\verb|Interaction|: After initializing HyperController, the algorithm will now interact with the training of the neural networks. For each iteration of training $t$, HyperController has two stages, \verb|Action Selection| and \verb|Update Models|. 

\begin{itemize}
    \item \verb|Action Selection|: For each $\zeta[i]$, $i \in [h]$, a prediction of each discretized hyperparameter value $\zeta_a$, $a \in [d]$ is given as $\hat{G}_{a}^t \left(\mathbf{c}_i\right)^\top \Xi_t\left(\mathbf{c}_i\right)$. Note that the vector $\Xi_t\left(\mathbf{c}_i\right)$ is the vector of previously observed rewards $X_{t-s},\dots,X_{t-1}$. 
    \item \verb|Update Models|: To update each model, HyperController first sets a tuple of previously chosen hyperparameter values $\zeta_{a_{t-s}},\dots,\zeta_{a_{t-1}}$ for $\zeta[i]$ as $\mathbf{c}_i$. Given $\mathbf{c}_i$ and the hyperparameter value $\zeta_{a_t}$ chosen at $t$, we add the iteration $t$ to the set $\mathcal{T}_{a_t\mid\mathbf{c}_i}^t$ to keep track of when we choose action $a_t$ given tuple $\mathbf{c}_i$. Next, we update the matrix $V_{a_t}^t\left(\mathbf{c}_i\right)$ and the model $\hat{G}_{a_t}^t\left(\mathbf{c}_i\right)$ based on \eqref{eq:V_def} and \eqref{eq:learn_G}, respectively. We use $B_{a_t}^t \left(\mathbf{c}_i\right)$ such that $\hat{G}_{a_t}^t\left(\mathbf{c}_i\right)$ can be updated each iteration $t$. The amount of computation work required for updating $V_{a_t}^t\left(\mathbf{c}_i\right)$ and the model $\hat{G}_{a_t}^t\left(\mathbf{c}_i\right)$ is $s^3$.  
    
\end{itemize}

To validate the performance of the method \eqref{eq:action_selection}, we will analyze it with respect to regret $R_n$ defined \eqref{eq:regret_def}, which is the cumulative difference between the highest possible value of $\psi\left(\zeta^*\left(t\right);\theta\left(t\right)\right)$ and the value $\psi\left(\zeta_{A_t};\theta\left(t\right)\right)$. The details for derivation of Theorem \ref{theorem:regret} can be found in the Appendix \ref{appendix:regret}. Details on the error of model $\hat{G}_{a}^t\left(\mathbf{c}_i\right)$ and its prediction error $X_t - \hat{G}_{a}^t\left(\mathbf{c}_i\right)^\top \Xi_t\left(\mathbf{c}_i\right)$ are in Appendix \ref{appendix:model_error}.

\begin{theorem}\label{theorem:regret}
    Let the reward $X_t$ be the output of the LGDS \eqref{eq:LGDS} based on equation \eqref{eq:LGDS_output}. Let $\hat{G}_{a}^t\left(\mathbf{c}_i\right)$ be learned accordingly to \eqref{eq:learn_G}. If actions are selected based on optimization problem \eqref{eq:action_selection}, then regret $R_n$ increases at the following rate with a probability of at least $1-13\delta$, $\delta \in (0,1)$:
    \begin{multline}\label{eq:regret_rate}
        R_n = \mathcal{O}\left(n\right) + \mathcal{O}\left(n\sqrt{\frac{\log\left(n\right)}{d}}\right) + \mathcal{O}\left(hd^{s+1} \sqrt{sn\log^3\left(n\right)}\right) \\ + \mathcal{O}\left(hd^{s+1} n\sqrt{s\log^3\left(n\right)}\right)
        + \mathcal{O}\left(hd^s n\sqrt{s\log\left(n\right)}\right) \\ + \mathcal{O}\left(hd^s n\sqrt{sn\log\left(n\right)}\right)
        + \mathcal{O}\left(kn\sqrt{\log\left(n\right)}\right) \\ + \mathcal{O}\left(k\sqrt{n}\right).
    \end{multline}
\end{theorem}

Theorem \ref{theorem:regret} proves that using action selection method \eqref{eq:action_selection} will lead to a regret \eqref{eq:regret_def} that will increase at most $\mathcal{O}\left(hd^s n\sqrt{sn\log\left(n\right)}\right)$. In the regret asymptotic rate \eqref{eq:regret_rate}, the term $\mathcal{O}\left(n\right)$ is from selecting the local optimal hyperparameter configuration and the term $\mathcal{O}\left(n\sqrt{\log\left(n\right)/d}\right)$ is from discretizing the hyperparameter space $\mathcal{Y}$. The rest of the terms are from modeling and prediction errors of $\hat{G}_{a}^t \left(\mathbf{c}_i\right)$. 

The regret rate in \eqref{eq:regret_rate} provides some intuition on how to set the parameters $s$ and $d$. In \eqref{eq:regret_rate}, regret increases exponentially with respect to $s$ and linearly with respect to $d$ with one term decreases $1/d$. Therefore, $s$ should be set very small---we propose a maximum of $3$---while $d$ should be set larger---we propose setting this value around $10$. 

This superlinear increase in regret should not be concerning as it has been proven in Theorem 4.1 of \cite{bogunovic2016time} that for any $f_t \sim \mathcal{G}\mathcal{P}\left(\mathbf{0},W\left(\cdot,\cdot\right)\right)$, then any algorithm must accumulate linear regret if $f_t$ satisfies the following properties 
\begin{itemize}
    \item $f_t$ is almost surely twice continuously differentiable. 
    \item $\sup \left\vert \partial f_t / \partial x_i\right\vert > L$ with a probability of most $a_1 e^{-(L/b_1)^2}$ for some $(a_1,b_1)$ and for all $L \geq 0$, $i \in [h]$.
    \item $\sup \left\vert \partial^2 f_t / \partial x_i \partial x_j\right\vert > L$ with a probability of most $a_2 e^{-(L/b_2)^2}$ for some $(a_2,b_2)$ and for all $L \geq 0$, $i,j \in [h]$. 
\end{itemize}

\subsection{Computational Efficiency}

One of the major advantages of using HyperController is its efficiency over other Gaussian Process optimization method. To better explain HyperController's efficiency, we first have to explain how Gaussian Processes are used. 

Let there be a cost function $f\left(x\right):\mathbb{R}^d \rightarrow \mathbb{R}$ that is \textit{unknown}. Assume for every value $x_t$ inputted into the cost function $f\left(x\right)$ we get the output $y_t$ that has the following expression
\begin{equation}
    y_t = f\left(x_t\right) \nonumber.
\end{equation}

We will assume that at iteration $t$ points $\left(x_0,y_0\right),\dots,\left(x_t,y_t\right)$ have been observed and we want to predict $f\left(x\right)$ for values $x_0',\dots,x_N'$. It is assumed that $y_0,\dots,y_t$ and $y_0',\dots,y_N'$ have the following \textit{a priori} distribution
\begin{equation}
    \begin{pmatrix}
        Y \\
        Y'
    \end{pmatrix} \sim \mathcal{N}\left(\begin{pmatrix}
        \mathbf{0} \\
        \mathbf{0}
    \end{pmatrix}, \begin{pmatrix}
        \Sigma_{Y} & \Sigma_{Y'Y}^\top \\
        \Sigma_{Y' Y} & \Sigma_{Y'}
    \end{pmatrix} + \sigma^2 I_{t + N}\right) \nonumber, 
\end{equation}
\begin{align}
    Y \triangleq \begin{pmatrix}
        y_0 & \dots & y_t
    \end{pmatrix}^\top,~~  & Y'\triangleq \begin{pmatrix}
        y_0' & \dots & y_N'
    \end{pmatrix}^\top \nonumber, 
\end{align}
where the covariance matrices $\Sigma_Y \in \mathbb{R}^{t \times t}$, $\Sigma_{Y'Y} \in \mathbb{R}^{N \times t}$, $\Sigma_{Y'} \in \mathbb{R}^{N \times N}$ are also assumed to be known. Therefore, if we want to predict the values $Y'$ using $Y$ then we have the following operation 
\begin{equation}
    \mathbb{E}\left[Y' \mid Y\right] = \Sigma_{Y'Y} \Sigma_Y^{-1} Y \nonumber. 
\end{equation}

Therefore, predicting $Y'$ using $Y$ involves inverting $\Sigma_Y$ which requires $\mathcal{O}\left(t^3\right)$ work for each iteration $t$. This operation has been improved to $\mathcal{O}\left(t \log^2\left(t\right)\right)$ work in \cite{ambikasaran2015fast}. However, when using HyperController, the operation only requires inverting the matrix $V_a^t\left(\mathbf{c}_i\right)\in \mathbb{R}^{s \times s}$ which is only $\mathcal{O}\left(s^3\right)$ operations where $s$ is at most $3$ (see discussion of regret in Section \ref{sec:algorithm}). The trade-off is sacrificing memory where there are now $hd^s$ models $\hat{G}_{a}^t\left(\mathbf{c}_i\right)$ to hold in memory. 

\section{Experiments}\label{sec:Experiments}

For the experimental results, we focus on online hyperparameter optimization for training RL neural networks on a variety of environments. In each of the tasks, we aimed to improve the algorithms in a single-machine/single GPU setting to test for computational efficiency. This allows us to consider scenarios where the machine learning models are too large to run multiple copies. The algorithms compared in all the environments are GP-UCB in \cite{srinivas2009gaussian}, PB2 in \cite{parker2020provably} (which is inspired by TV-GP-UCB in \cite{bogunovic2016time}), the randomized algorithm proposed by Bergstra et al. \cite{bergstra2012random} called Random, HyperBand \cite{li2017hyperband}, a method uses a constant random configuration set at the beginning of training called Random Start, and our proposed Algorithm \ref{alg:HyperController}, HyperController, from Section \ref{sec:algorithm}. For each task, 10 random seeds were used where each algorithm used the same set of seeds. 

The RL neural network was trained using Proximal Policy Optimization (PPO) \cite{schulman2017proximal}. The observations the RL neural networks were trained on were the observation space of each environment which can include variables such as velocity and position. The environments, which are a part of the OpenAI Gymnasium \cite{brockman2016openaigym}, are: 
\begin{itemize}
    \item \verb|HalfCheetah-v4|
    \item \verb|BipedalWalker-v3|
    \item \verb|Pusher-v4|
    \item \verb|InvertedDoublePendulum-v4|
    \item \verb|Reacher-v4|
\end{itemize}

In this setup, the hyperparameters tuned include the batch size (how many frames of the RL neural network interaction with the environment were combined to train the network in one iteration), the clip parameter in PPO, the lambda decay parameter in Generalized Advantage Estimation \cite{schulman2015high}, and the ADAM optimizer's learning rate \cite{kingma2014adam}. The cost function maximized by the hyperparameter optimizers was the change in the collected reward. 

In Figure \ref{figure:Reward_RL}, each line represents the median evaluation reward as a function of wall-clock time. To ensure comparability, we only included the wall clock time elapsed during the hyperparameter search. Note that the $x$-axis of Figure \ref{figure:Reward_RL} is in a logarithmic scale of minutes. The figure shows that Random, Random Start, HyperBand, and HyperController all finish the $1000$ iterations of training under $0.1$ minutes. However, GP-UCB and PB2 take longer than $10$ minutes. In the left plots, HyperController achieves the highest median training reward in all environments besides the \verb|InvertedDoublePendulum-v4|. For \verb|InvertedDoublePendulum-v4|, GP-UCB and PB2 perform better but require significantly more time (an additional $10$ minutes).

In the right plots, each line represents the median evaluation reward sum as a function of wall-clock time. HyperController outperforms Random, Random Start, and HyperBand in terms of median evaluation reward sum. For GP-UCB and PB2, HyperController achieves higher median evaluation reward sum within $0.1$ minutes in \verb|HalfCheetah-v4| and \verb|InvertedDoublePendulum-v4|. Finally, in \verb|BipedalWalker-v3|, \verb|Pusher-v4|, and \verb|Reacher-v4|, HyperController obtains the highest median evaluation reward sum, even after $10$ minutes of using GP-UCB and PB2. 

In Figure \ref{figure:Reward_RL_Stability}, we present a box plot of the final training reward (left plot) and evaluation reward sum (right plot) at iteration $t = 1000$. We focus on the final iteration $t = 1000$ since the method with the highest final training reward and evaluation reward sum values indicates the method's ability to maintain its improvement of the RL neural network until the end of training, i.e. stabilize training. 

In the box plot, the edges of the boxes represent the first and third quantiles and the middle of the box is the median. In the plots to the left (training rewards), we can observe that HyperController achieves the highest median training rewards except in the \verb|InvertedDoublePendulum-v4| environment, where GP-UCB and PB2 obtain higher training reward. This is observed in the plots on the right (evaluation reward sum) as well. It is notable that the first quantile of HyperController are higher than the medians of the other methods in evaluation reward sum. This is with the exception of PB2 and GP-UCB in the \verb|InvertedDoublePendulum-v4| environment and the \verb|HalfCheetah-v4| environment. 

The results in Figures \ref{figure:Reward_RL} and \ref{figure:Reward_RL_Stability} indicate that HyperController is a fast and stable hyperparameter controller, demonstrating that the neural network maintains performance improvements efficiently in terms of wall-clock time. Based on the results in Figure \ref{figure:Reward_RL}, HyperController achieves the highest median training reward and evaluation reward sum in four out of the five environments. Moreover, HyperController achieves these results efficiently, offering advantages such as faster training, quicker deployment, shorter turnaround times for experimenting with different RL neural networks, and easier fine-tuning during the transition from the training environment to the deployed environment. Finally, Figure \ref{figure:Reward_RL_Stability} further highlights HyperController's consistent ability to maintain improvements in the RL neural networks until the final iteration.  

\begin{figure}[ht]
    \centering
    \includegraphics[width=\linewidth]{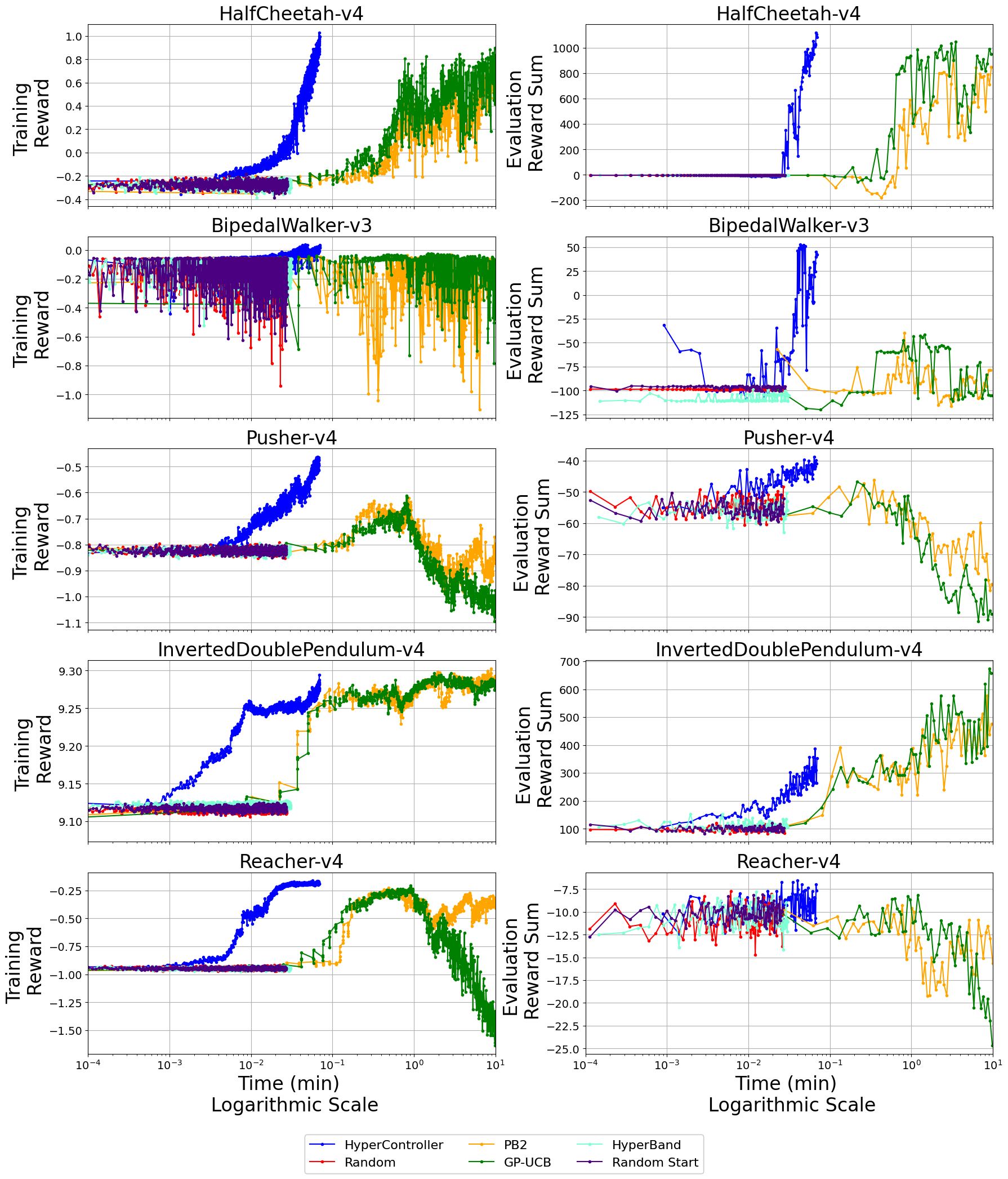}
    \caption{The plots are the median training rewards (left plots) and median evaluation reward sum (right plots) for each method over are a function of wall clock time.}
    \label{figure:Reward_RL}
\end{figure}

\begin{figure}[ht]
    \centering
    \includegraphics[width=\linewidth]{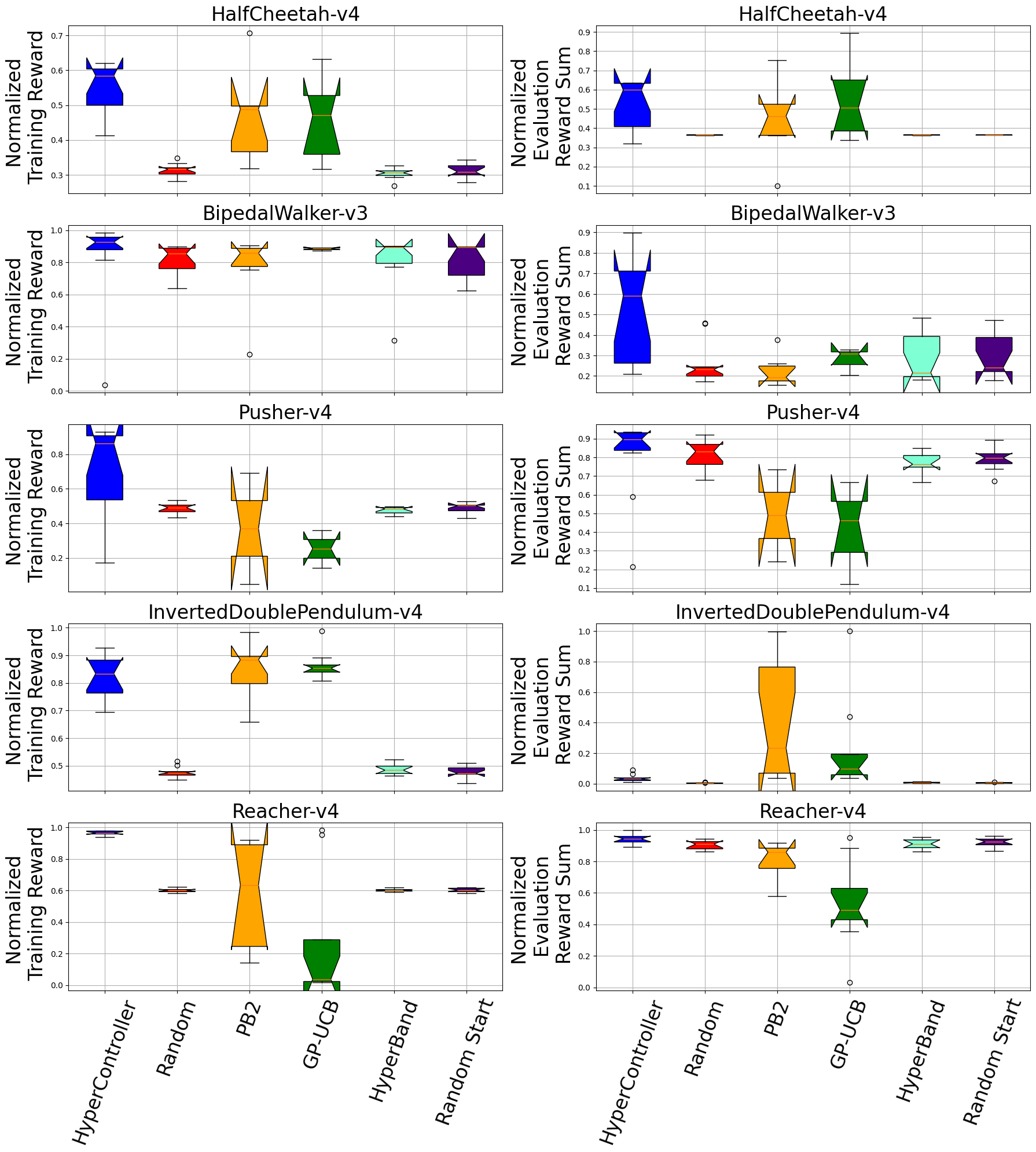}
    \caption{The plots are the median training rewards (left plots) and median evaluation reward sum (right plots) at the final training iteration.}
    \label{figure:Reward_RL_Stability}
\end{figure}

\section{Conclusion}\label{sec:Conclusion}

In this paper, we introduced a method that learns a more efficient representation of the Bayesian model in Bayesian optimization. The core contribution of this approach is a computationally fast algorithm for hyperparameter optimization. To measure the performance of the algorithm, we provided a theoretical bound on regret, which is the cumulative difference between the highest possible reward collected minus the reward collected from the proposed method. We then provided numerical results of the method with respect to a number of well-known methods. Here, we show that the proposed method is able to train the network more efficiently. Future directions of this work is to implement this method as a controller utilized during deployment of the RL neural networks.

\appendix

\subsection{Derivation of the LGDS}\label{appendix:derivation}

In this paper, we avoided adding the derivations in the main section to help the flow of discussion. Therefore, discussions of how we derived the LGDS in \eqref{eq:LGDS}, Section \ref{sec:Problem Formulation}, for modeling dynamics of $\psi\left(\zeta;\theta\left(t\right)\right)$ within \eqref{eq:optimization_problem} are provided here. First, it is assumed that there exists a function $f_t \in L^2\left(\mathbb{R}^h \rightarrow \mathbb{R}\right)$ that models the changes in $\psi\left(\zeta;\theta\left(t\right)\right)$ as an infinite-dimensional Linear Gaussian Dynamical System (LGDS)
\begin{equation}\label{eq:psi_dynamics}
    f_{t+1} = \Omega f_t + \omega_t , 
\end{equation}
where $\Omega$ is an infinite-dimensional state matrix, i.e. $\Omega \in \left\{\Omega\left(\cdot ,\cdot\right):\mathcal{Y} \times \mathcal{Y} \rightarrow \mathbb{R}\right\}$. The infinite-dimensional vector $\omega_t \in \left\{\omega_t:\mathcal{Y} \rightarrow \mathbb{R}\right\}$ is a Gaussian Process with a covariance function $W\left(\cdot ,\cdot\right) \in \left\{W\left(\cdot ,\cdot\right):\mathcal{Y} \times \mathcal{Y} \rightarrow \mathbb{R}\right\}$, i.e. $\omega_t \sim \mathcal{G}\mathcal{P}\left(\mathbf{0},W\left(\cdot ,\cdot\right)\right)$. Whenever a hyperparameter $\zeta\left(t\right) \in\mathcal{Y}$ is chosen, the loss $\psi\left(\zeta\left(t\right);\theta\left(t\right)\right)$ is observed $X_t$ such that 
\begin{equation}\label{eq:psi_output}
    X_t = f_t\left(\zeta\left(t\right)\right) + \eta_t, 
\end{equation}
where $\eta_t \sim \mathcal{N}\left(0,\sigma^2\right)$ is error that is not modeled through $f_t\left(\zeta\left(t\right)\right)$. The following assumptions are imposed on the LGDS \eqref{eq:psi_dynamics} and its output \eqref{eq:psi_output}.

\begin{assumption}\label{assum:unknown}
    The LGDS state matrix $\Omega$ and the covariance matrix $W$ are assumed to be \textit{unknown}. 
\end{assumption}

\begin{assumption}\label{assum:continuity}
    The functions $\Omega\left(\zeta_i,\zeta_j\right)$ and $W\left(\zeta_i,\zeta_j\right)$ are Lipshitz continuous with respect to $\zeta_i$ and $\zeta_j$ with constants $L_{\Omega} > 0$ and $L_W > 0$, respectively. 
\end{assumption}

\begin{assumption}\label{assum:contracting}
    Define the product $\Omega f_t$ as 
    \begin{equation}
        \Omega f_t = \int_{} \Omega\left(\zeta,\zeta'\right)f_t\left(\zeta'\right) d\zeta' \nonumber.  
    \end{equation}

    It is assumed that the product $\Omega f_t$ is strongly contracting with respect to $f_t$ with constant $C \in (0,1)$. 
\end{assumption}

\begin{assumption}\label{assum:definiteness}
    The matrix $W\left(\cdot,\cdot\right)$ is a positive semi-definite kernel, i.e. 
    \begin{equation}
        \sum_{i=1}^d \sum_{j=1}^d c_i c_j W\left(\zeta_i,\zeta_j\right) d\zeta_i d\zeta_j \geq 0 \nonumber, 
    \end{equation}
    for all $d \in \mathbb{N}$, $\zeta_1,\dots,\zeta_d \in \mathcal{Y}$, $c_i,c_j \in \mathbb{R}$ \cite{pmlr-vR5-hein05a}.
\end{assumption}

\begin{assumption}\label{assum:continuous_controllability}
    Let there be the function $\Psi_t:\mathcal{Y} \times \mathcal{Y} \rightarrow \mathbb{R}$ which is defined to be
    \begin{equation}
        \Psi_t\left(\zeta_i,\zeta_j\right) \triangleq \mathbb{E}\left[f_t\left(\zeta_i\right)f_t\left(\zeta_j\right)\right] \label{eq:psi_covariance}. 
    \end{equation}
    
    Using the definition $\Psi_t$ and the dynamics of $f_t$, we can write the iteration for $\Psi_t$ below
    \begin{equation}
        \Psi_{t+1} = \Omega \Psi_t \Omega' + W \label{eq:covariate_continuous_iteration}, 
    \end{equation}
    where $\Omega \Psi_t \Omega'$ is defined as 
    \begin{multline}\label{eq:omega_matrix_product}
        \left(\Omega \Psi_t \Omega'\right)\left(\zeta_i,\zeta_j\right) \\ =  
        \iint \Omega\left(\zeta_i,\tilde{\zeta}_i\right) W\left(\tilde{\zeta}_i,\tilde{\zeta}_j\right) \Omega\left(\zeta_j,\tilde{\zeta}_j\right) d\tilde{\zeta}_i d\tilde{\zeta}_j . 
    \end{multline}

    It is assumed for any $t \in [n]$, $\Psi_t$ is a positive definite kernel, i.e.
    \begin{equation}
        \sum_{i=1}^d \sum_{j=1}^d c_i c_j \Psi_t\left(\zeta_i,\zeta_j\right) d\zeta_i d\zeta_j > 0 \nonumber, 
    \end{equation}
    for all $d \in \mathbb{N}$, $\zeta_1,\dots,\zeta_d \in \mathcal{Y}$, $c_i,c_j \in \mathbb{R}$. 
    
\end{assumption}

Assumptions \ref{assum:unknown}-\ref{assum:continuous_controllability} are generalizations of the assumptions made in \cite{bogunovic2016time} and \cite{parker2020provably}. In \cite{bogunovic2016time} and \cite{parker2020provably}, the authors assume that $\Omega \equiv \sqrt{1-\alpha}$ and $\omega_t \sim \mathcal{G}\mathcal{P}\left(\mathbf{0},\alpha W\left(\cdot,\cdot\right)\right)$ where $\alpha\in [0,1]$ and $W\left(\cdot,\cdot\right)$ are both \textit{known}. 

When using the formulation of $f_t$, we are treating it as an infinite-dimensional vector where the components of the vector are at values $\zeta \in \mathcal{Y}$. This helps us avoid restrictions on how the function $f_t$ is structured at a given time and instead focus on how it changes over time $t$. For example, the function $f_t$ can still be nonlinear even if it dynamics \eqref{eq:psi_dynamics} are linear.


\subsection{Discretization}\label{appendix:discretize}

To improve computational efficiency of the method, we will discretize the space $\mathcal{Y}$. As a consequence, we will sacrifice the accuracy of $f_t$. To analyze this loss in accuracy, we will prove the continuity of $f_t$ and its statistics. By proving the continuity properties, we can analyze how much accuracy is lost when discretizing the space. The theorem below will prove that $f_t$ is almost surely continuous. 

\begin{theorem}\label{theorem:continuity}
    Impose Assumptions \ref{assum:continuity}-\ref{assum:continuous_controllability}. Therefore, the infinite-dimensional vector $f_t$ in \eqref{eq:psi_dynamics} is almost surely continuous. In other words, there exists a $\delta >0$ such that for all $\zeta_1,\zeta_2 \in \mathcal{Y}$ such that $\left\Vert \zeta_1 - \zeta_2\right\Vert_2 < \delta$
    \begin{equation}
        P\left(\left\vert f_t\left(\zeta_1\right) - f_t\left(\zeta_2\right)\right\vert > 0\right) = 0\nonumber. 
    \end{equation}

    In addition, the covariance function $\Psi_t$, which is defined to be \eqref{eq:psi_covariance}, is Lipshitz continuous with respect to $\zeta_i$ and $\zeta_j$. 
    
\end{theorem}

Theorem \ref{theorem:continuity} proves that $\Psi_t$ defined \eqref{eq:psi_covariance} maintains Lipshitz continuity for any $t$, also implying that $f_t$ is almost surely continuous. The advantage of this is that we can directly measure how discretizing the space of $\mathcal{Y}$ impacts performance in a later section. To address discretization of the hyperparameter space $\mathcal{Y}$, we define the set $\mathcal{D} = \mathcal{D}_1 \times \cdots \times \mathcal{D}_{h}$ where $\mathcal{D}_i$ is defined as
\begin{equation}\label{eq:discretized_space}
    \mathcal{D}_i \triangleq \left\{\zeta_{A[i]}[i] = \zeta_a \mid \left(\dots,\zeta_a,\dots\right) \in \mathcal{Y}, a \in [d]\right\}. 
\end{equation}

Correspondingly, let $z_t \in \mathbb{R}^{d^h}$ be the discretized version of $f_t$, where each component of $z_t$ models the response of $f_t\left(\zeta_A\right)$, $\zeta_A \in\mathcal{D}$ (see \eqref{eq:state_z_t_definition}).

\subsection{Modeling}\label{appendix:Modeling}

We can approximate the dynamics of $f_t$ defined as \eqref{eq:psi_dynamics}, which is an infinite dimensional vector in the $L^2\left(\mathbb{R}^h \rightarrow \mathbb{R}\right)$ space that evolves linearly, with a finite dimensional vector $z_t \in \mathbb{R}^{d^h}$ sampled from a LGDS. Given the set $\mathcal{D}$ defined as \eqref{eq:discretized_space} and the vector $z_t \in \mathbb{R}^{d^h}$ \eqref{eq:state_z_t_definition}, $z_t$ evolves based on the following LGDS \eqref{eq:LGDS} where the state $z_t \in \mathbb{R}^{d^h}$, $\Gamma \in \mathbb{R}^{d^h \times d^h}$ is a discretization of the unknown matrix $\Omega$, and $\xi_t \in \mathbb{R}^{d^h}$ is a discretization of the infinite-dimensional vector $\omega_t$, i.e. $\xi_t \in \mathcal{N}\left(\mathbf{0},Q\right)$, $Q \succeq \mathbf{0}$ where $Q[i,j] = W\left(\zeta_i,\zeta_j\right)$, $\zeta_i,\zeta_j \in \mathcal{Y}$. The response of $X_t = f_t\left(\zeta_A\right)$, given that $\zeta_A\in \mathcal{D}$ has been chosen, is expressed as \eqref{eq:LGDS_output} where component of the vector $\mathbf{e}_A$ that is associated with $\psi\left(\zeta_{A[1]}[1],\dots,\zeta_{A[h]}[h];\theta\left(t\right)\right)$ in $z_t$ defined as \eqref{eq:state_z_t_definition} is 1 and the rest of the components are 0. The scalar $\eta_t \in \mathbb{R}$ is modeled as zero-mean Gaussian noise with a variance of $\sigma^2$, i.e. $\eta_t \in \mathcal{N}\left(0,\sigma^2\right)$. Given the Assumptions \ref{assum:continuity}-\ref{assum:continuous_controllability} posed for $\Omega$ and $W$, we have the following properties for the LGDS \eqref{eq:LGDS}.

\begin{lemma}\label{lemma:LGDS_properties}
    Let there be the LGDS \eqref{eq:LGDS} and the Assumptions \ref{assum:continuity}-\ref{assum:continuous_controllability} posed for $\Omega$ and $W$. Also assume that $\Psi_t$ \eqref{eq:psi_covariance} defined in Theorem \ref{theorem:continuity} is positive definite. Therefore, the matrix $\Gamma$ is Schur, i.e. $\rho\left(\Gamma\right) < 1$, and the matrix pair $\left(\Gamma,Q^{1/2}\right)$ is  controllable. 
\end{lemma}

The derivations above lead us to the following problem we aim to solve in this paper. Recall that the problem is to search for the hyperparameter $\zeta_A \in \mathcal{D}$ that maximizes $\psi\left(\zeta_A;\theta_t\right)$ as efficiently as possible at each time $t$. Therefore, based on the discretization of $f_t$ as $z_t$ in \eqref{eq:LGDS} with output $X_t$ in \eqref{eq:LGDS_output}, we want to develop a method that finds the sequence of vectors $\mathbf{e}_{A_0}, \dots, \mathbf{e}_{A_n}$ (which are mapped to hyperparameters $\zeta \in \mathcal{D}$) that maximizes cumulative sum $\sum_{t=1}^n X_t \equiv \sum_{t=1}^n f_t\left(\zeta_{A_t}\right)$. To measure performance, we will analyze \textit{pseudo-regret} (referred to as regret for simplicity), which is the cumulative difference between the highest possible value of $f_t\left(\zeta^*\left(t\right)\right)$ and the value $f_t\left(\zeta_{A_t}\right)$ based on the chosen hyperparameter $\zeta_{A_t} \equiv \zeta\left(t\right)$ (see regret in \eqref{eq:regret_def}). To address this problem, we will review how this problem can be approached as a stochastic multi-armed bandit environments. Since this environment is a LGDS, we will use inspiration from the results of \cite{gornet_pmlr} for the proposed method. We will then provide an analysis of the method's performance with respect to regret $R_n$ defined in \eqref{eq:regret_def}.

\subsection{A Representation for Predicting the Output of the Hyperparameter Cost Function}\label{appendix:representation}

Based on the dynamics of $z_t$ in \eqref{eq:LGDS} and the output $X_t$ in \eqref{eq:LGDS_output}, if one knew all the parameters of \eqref{eq:LGDS} and \eqref{eq:LGDS_output}, then the optimal one-step predictor for $X_t$ is the Kalman filter:
\begin{equation}\label{eq:Kalman_Filter}
    \begin{cases}
        \hat{z}_{t+1|t} & = \Gamma \hat{z}_{t|t-1} + \Gamma K_t \left(X_t - \left\langle \mathbf{e}_{A_t},\hat{z}_{t|t-1}\right\rangle\right) \\
        X_t & = \left\langle \mathbf{e}_{A_t},\hat{z}_{t|t-1}\right\rangle + \varepsilon_{A_t}^t \\
        K_t & = P_{t|t-1} \mathbf{e}_{A_t}\left(\mathbf{e}_{A_t}^\top P_{t|t-1} \mathbf{e}_{A_t} + \sigma^2\right)^{-1} 
    \end{cases}, 
\end{equation}
where the state prediction $\hat{z}_{t|t-1} \triangleq \mathbb{E}\left[z_t \mid \mathcal{F}_{t-1}\right]$ and $\mathcal{F}_{t-1}$ is the sigma algebra of previous observed rewards $X_0,\dots,X_{t-1}$. The matrix $P_{t|t-1}$ is the covariance of the error $e_{t|t-1} \triangleq z_t - \hat{z}_{t|t-1}$ which is the output the following difference Riccati equation
\begin{equation}\label{eq:DARE}
    P_{t+1|t} \triangleq g\left(P_{t|t-1},\mathbf{e}_{A_t}\right), 
\end{equation}
\begin{multline}\label{eq:g_def}
    g\left(P,\mathbf{e}_A\right) \triangleq \Gamma P \Gamma^\top + Q \\ - \Gamma P \mathbf{e}_A\left(\mathbf{e}_A^\top P \mathbf{e}_A + \sigma^2\right)^{-1} \mathbf{e}_A^\top P \Gamma^\top . 
\end{multline}

The random variable $\varepsilon_{A_t}^t \sim \mathcal{N}\left(0,\mathbf{e}_{A_t}^\top P_{t|t-1} \mathbf{e}_{A_t}\right)$. Since most of the parameters are unknown in \eqref{eq:LGDS} and \eqref{eq:LGDS_output}, we cannot directly use the Kalman filter \eqref{eq:Kalman_Filter}. In addition, since the matrices $P_{t|t-1}$ and $K_t$ evolve based on the sequence of chosen actions $\mathbf{e}_{A_0},\dots,\mathbf{e}_{A_{t-1}}$, identifying the terms of \eqref{eq:Kalman_Filter} is intractable. Therefore, using the same thought process presented in \cite{gornet_pmlr}, we will identify a modified Kalman filter instead of identifying the terms in the Kalman filter in \eqref{eq:Kalman_Filter}. The modified Kalman filter is posed in Theorem \ref{theorem:code_seq} below. 

\begin{theorem}\label{theorem:code_seq}
    Assume that there exists $P_{\overline{A}}$ such that for every $P_A$ that solves the following Algebraic Riccati equation
    \begin{equation}
        P_A = g\left(P_A,\mathbf{e}_A\right)\nonumber, 
    \end{equation}
    the matrix inequality $P_{\overline{A}} \succeq P_A$ is satisfied. We define the following modified Kalman filter
    \begin{align}\label{eq:modified_Kalman_filter}
        \begin{cases}
            \hat{z}_{t+1}' & = \Gamma \hat{z}_t' + \Gamma L_{A_t} \left(X_t - \left\langle \mathbf{e}_{A_t}, \hat{z}_{t}'\right\rangle \right) \\
            X_t & = \left\langle \mathbf{e}_{A_t},\hat{z}_{t}'\right\rangle + \gamma_{A_t}^t
        \end{cases}, \nonumber \\
        ~~ L_{A_t} \triangleq P_{\overline{A}} \mathbf{e}_{A_t}\left(\mathbf{e}_{A_t}^\top P_{\overline{A}} \mathbf{e}_{A_t} + \sigma^2\right)^{-1} , 
    \end{align}
    where $\gamma_{A_t}^t \triangleq X_t - \left\langle \mathbf{e}_{A_t},\hat{z}_{t}'\right\rangle \sim \mathcal{N}\left(0,\mathbf{e}_{A_t}^\top P_t' \mathbf{e}_{A_t} + \sigma^2\right)$ and $P_t' \triangleq \mathbb{E}\left[\left(z_t - \hat{z}_{t}'\right)\left(z_t - \hat{z}_{t}'\right)^\top \mid \mathcal{F}_{t-1}\right]$. The filtration $\mathcal{F}_t$ is the sigma algebra generated by previous observations $X_0, \dots, X_t$. It is proven for the modified Kalman filter \eqref{eq:modified_Kalman_filter} that 1) the matrix $\Gamma - \Gamma L_{A_t}\mathbf{e}_{A_t}^\top$ is stable and 2) the variance of the residual $\mbox{Var}\left(\gamma_{A_t}^t\right)$ is bounded.
\end{theorem}

Based on Theorem \ref{theorem:code_seq}, it is proven that the reward $X_t$ can be expressed as the output of a modified Kalman filter \eqref{eq:modified_Kalman_filter} where the matrices $P_{\overline{A}}$ and $L_{A_t}$, $A_t \in [d]^h$, that substituted the matrices $P_{t|t-1}$ and $K_t$, $t\in [n]$, in the Kalman filter \eqref{eq:Kalman_Filter} are now constant.

\subsection{Alternating Directions}\label{appendix:alternating}

In the prior subsection, we showed that modeling $f_t$ as a discretized form $z_t$ in \eqref{eq:LGDS} with output $X_t$ \eqref{eq:LGDS_output}. A major issue is the following. If we have $h$ hyperparameters where each hyperparameter value $\zeta_{A[i]}[i]$, $i \in [h]$, is discretized into $d$ values (i.e. $\zeta_{A[i]}[i] \in \left\{\zeta_1[i],\dots,\zeta_d[i]\right\}$), then $z_t \in \mathbb{R}^{d^{h}}$, implying that there are $d^{h}$ hyperparameter configurations. This will lead to a number of issues we have to address. The issues are listed as follows: 
\begin{itemize}
    \item \textbf{Issue 1:} The number of configurations $d^h$ can be larger than the total number of rounds $n$. Therefore, it would be infeasible for the learner to explore all the configurations $\mathbf{e}_A$ to learn which configuration will return the highest reward $X_t$. 
    \item \textbf{Issue 2:} The required memory for storing all the identified representations of the modified Kalman filter \eqref{eq:modified_Kalman_filter} is too high (at least $d^h$ values need to be saved).
\end{itemize}

To approach these issues, we will greedily search for each hyperparameter value $\zeta_{A[i]}[i] = \zeta_a$, $i \in [h]$, instead of search for the value $\zeta_A$ that globally optimizes the function. This will lead to an alternate expression of $X_t$ which we will have to derive before identifying a representation of the modified Kalman filter \eqref{eq:modified_Kalman_filter}. First, recall the definition of $z_t$ in \eqref{eq:state_z_t_definition}. Let us assume that the proposed strategy ``thinks'' action $\mathbf{e}_{\mathbf{1}}$ was chosen when action $\mathbf{e}_A$, $A \neq \mathbf{1}$ was chosen. Therefore, the reward $X_t$ \eqref{eq:LGDS_output} is expressed as follows:
\begin{equation}
    X_t = \left\langle \mathbf{e}_{\mathbf{1}}, z_t \right\rangle + \left\langle \mathbf{e}_A - \mathbf{e}_{\mathbf{1}}, z_t \right\rangle + \eta_t \label{eq:wrong_output}. 
\end{equation}

Based on above, we are now able to provide the methodology for identifying the representation of the modified Kalman filter in the following subsection. 

\subsection{Identifying the Representation}\label{appendix:identify}


Given the results from the previous subsections, we can now learn a representation for predicting the objective function response with respect to each hyperparameter. First, we will use the expression of $X_t$ \eqref{eq:wrong_output} to learn a representation of the modified Kalman filter derived in Theorem \ref{theorem:code_seq}. Second, let there be the sequence $\mathbf{c}_i = \left(\zeta_{a_{t-1}}, \dots, \zeta_{a_{t-s}}\right)$ for $\zeta_{A[i]}[i]$, $i \in [h]$. Using the modified Kalman filter \eqref{eq:modified_Kalman_filter} and the output $X_t$ in \eqref{eq:wrong_output}, we can express the reward $X_t$ as a linear combination of observed rewards $X_{t-s},\dots,X_{t-1}$ where the linear parameters are based on the modified Kalman filter \eqref{eq:modified_Kalman_filter} matrices. 
\begin{multline}
    X_t = \mathbf{e}_A^\top\left(\prod_{\tau=1}^{s-1}\left(\Gamma - \Gamma L_{A_{t-\tau}} \right)\right) \Gamma L_{A_{t-s}}X_{t-s} + \dots \\ + \mathbf{e}_A^\top \Gamma L_{A_{t-1}}X_{t-1} \\
    + \left(\mathbf{e}_{A'} - \mathbf{e}_A\right)^\top\left(\prod_{\tau=1}^{s-1}\left(\Gamma - \Gamma L_{A_{t-\tau}} \right)\right) \Gamma L_{A_{t-s}}X_{t-s} + \dots \\ + \left(\mathbf{e}_{A'} - \mathbf{e}_A\right)^\top \Gamma L_{A_{t-1}}X_{t-1} \\ + \mathbf{e}_{A'}^\top \left(\prod_{\tau=1}^{s}\left(\Gamma - \Gamma L_{A_{t-\tau}} \right)\right) \hat{z}_{t-s}' + \gamma_{A'}^t \nonumber, 
\end{multline}
\begin{equation}\label{eq:linear_model}
    \Rightarrow X_t = G_{a}\left(\mathbf{c}_i\right)^\top \Xi_t\left(\mathbf{c}_i\right) + \phi_{A'}^t + \beta_{A'}^t + \gamma_{A'}^t, 
\end{equation}
where $\Xi_t\left(\mathbf{c}_i\right)$ is defined in \eqref{eq:xi_def} and $G_{a}\left(\mathbf{c}_i\right)$ , $\phi_{A'}^t$, and $\beta_{A'}^t$ are defined to be 
\begin{align}
    & G_{a}\left(\mathbf{c}_i\right) \nonumber \\
    & ~~~~\triangleq \Bigg(
    \begin{matrix}
        \mathbf{e}_A^\top\left(\prod_{\tau=1}^{s-1}\left(\Gamma - \Gamma L_{A_{t-\tau}} \right)\right) \Gamma L_{A_{t-s}} & 
    \end{matrix} \nonumber \\
    &  ~~~~~~~~~~~~~~~~~~~~~~~~~~~~~~~~~~~ 
    \begin{matrix}
        \dots &\mathbf{e}_A^\top \Gamma L_{A_{t-1}}
    \end{matrix}\Bigg) \in \mathbb{R}^s \label{eq:G_def} \\
    & \phi_{A'}^t  \triangleq \left(G_{a'}\left(\mathbf{c}_i\right)-G_{a}\left(\mathbf{c}_i\right)\right)^\top \Xi_t\left(\mathbf{c}\right) \in \mathbb{R} \label{eq:phi_def} \\
    & \beta_{A'}^t  \triangleq \mathbf{e}_{A'}^\top \left(\prod_{\tau=1}^s \left(\Gamma - \Gamma L_{A_{t-\tau}} \mathbf{e}_{A_{t-\tau}}^\top\right) \right) \hat{z}_{t-s}' \in \mathbb{R} \label{eq:beta_def}. 
\end{align}

The vectors $G_{a}\left(\mathbf{c}_i\right)$, $a,a'$, are unique for every tuple $\mathbf{c}_i$ (implying that there are $d^s$ vector $G_{a}\left(\mathbf{c}_i\right)$ to identify). In addition, since $\gamma_A^t$ is a zero-mean Gaussian random variable with a finite variance, identifying $G_{a}\left(\mathbf{c}_i\right)$ in the expression \eqref{eq:linear_model} is a linear least squares problem with an unknown bias $\phi_{A'}^t + \beta_{A'}^t$. The impact of the bias $\phi_{A'}^t + \beta_{A'}^t$ is analyzed through Theorems \ref{theorem:model_error} and \ref{theorem:extended_regret}. Therefore, the agent has to learn the linear predictor for each action $a \in [d]$. The identification of the vector $G_{a}\left(\mathbf{c}_i\right) \in \mathbb{R}^s$ is based on the discussion in Section \ref{sec:representation}. For the remainder of the appendix, we will provide the proofs for various lemmas and theorems mentioned in the paper. 

\subsection{Minor Details about the Numerical Results}

Table \ref{tab:compact_scores} shows the number of seeds (out of $10$) for which each method completed all $1000$ training iterations without returning \verb|NaN| values (denoted as success rate). For the implementation of PB2 and GP-UCB, we adapted the code posted in \url{https://github.com/ray-project/ray/tree/master/python/ray/tune}.

\begin{table}[htbp]
    \centering
    \scriptsize
    \caption{Per-environment success rate \\(HC: HyperController, R: Random, GP-UCB: GP-UCB, HB: HyperBand, RS: Random Start).}
    \label{tab:compact_scores}
    \begin{tabular}{lcccccc}
    \hline
    Env & HC & R & PB2 & GP-UCB & HB & RS \\
    \hline
    \verb|HalfCheetah-v4|       & 10 & 10 & 10 & 10 & 10 & 10 \\
    \verb|BipedalWalker-v3|     & 10 & 10 &  6 &  3 & 10 & 10 \\
    \verb|Pusher-v4|            & 10 & 10 &  7 & 10 & 10 & 10 \\
    \verb|InvDoublePendulum-v4| & 10 & 10 &  9 &  8 & 10 & 10 \\
    \verb|Reacher-v4|           &  9 &  9 &  7 &  8 &  9 & 10 \\
    \hline
    \end{tabular}
\end{table}

\subsection{Proof for Theorem \ref{theorem:continuity}}

\begin{proof}
    First, let there be the function $\Psi_t:\mathcal{Y} \times \mathcal{Y} \rightarrow \mathbb{R}$ which is defined to be \eqref{eq:psi_covariance}. Using the definition $\Psi_t$ and the dynamics of $f_t$, we can write the iteration for $\Psi_t$ as \eqref{eq:covariate_continuous_iteration} where $\Omega \Psi_t \Omega'$ is defined as \eqref{eq:omega_matrix_product}. According to a theorem in \cite{garnett2023bayesian}, if the function $K\left(\cdot,\cdot\right): \mathcal{Y} \times \mathcal{Y} \rightarrow \mathbb{R}$ is Lipshitz continuous and $\mathcal{Y}$ is compact, then the sampled function $f \sim \mathcal{G}\mathcal{P}\left(\mathbf{0},K\left(\cdot,\cdot\right)\right)$ is almost surely continuous. Therefore, if we prove that $\Psi_t$ is Lipshitz continuous, then $f_t$ is continuous. First, the iteration \eqref{eq:covariate_continuous_iteration} can be expressed as
    \begin{equation}
        \Psi_{t+1} = \Omega\left(\Omega \Psi_{t-1} \Omega' + W\right)\Omega' \nonumber,
    \end{equation}
    implying that we need to prove that $\Omega W \Omega'$ and $\Omega \Psi_{t-1} \Omega'$ are Lipshitz continuous. To address this, we will first prove that the product $\Omega W$ is Lipshitz continuous. Note the following
    \begin{align}
        & \left\vert \int \Omega\left(\zeta_i,\zeta'\right)W\left(\zeta',\zeta_j\right) - \Omega\left(\tilde{\zeta}_i,\zeta'\right)W\left(\zeta',\zeta_j\right)d\zeta'\right\vert \nonumber \\
        & ~~ = \left\vert \int \left(\Omega\left(\zeta_i,\zeta'\right) - \Omega\left(\tilde{\zeta}_i,\zeta'\right)\right)W\left(\zeta',\zeta_j\right)d\zeta'\right\vert \nonumber \\
        & ~~ \overset{(a)}{\leq}  \int \left\vert\left(\Omega\left(\zeta_i,\zeta'\right) - \Omega\left(\tilde{\zeta}_i,\zeta'\right)\right)W\left(\zeta',\zeta_j\right)\right\vert d\zeta' \nonumber \\
        & ~~ \overset{(b)}{\leq}  \int \left\vert \Omega\left(\zeta_i,\zeta'\right) - \Omega\left(\tilde{\zeta}_i,\zeta'\right)\right\vert \left\vert W\left(\zeta',\zeta_j\right)\right\vert d\zeta' \nonumber \\
        & ~~ \overset{(c)}{\leq}  \int L_{\Omega} \left\Vert \zeta_i - \tilde{\zeta}_i\right\Vert_2 \left\vert W\left(\zeta',\zeta_j\right)\right\vert d\zeta' \nonumber,
    \end{align}
    \begin{multline}
        \left\vert \int \Omega\left(\zeta_i,\zeta'\right)W\left(\zeta',\zeta_j\right) - \Omega\left(\tilde{\zeta}_i,\zeta'\right)W\left(\zeta',\zeta_j\right)d\zeta'\right\vert \\ \leq
        L_{\Omega} \left\Vert \zeta_i - \tilde{\zeta}_i\right\Vert_2 \int \left\vert W\left(\zeta',\zeta_j\right)\right\vert d\zeta' \nonumber,
    \end{multline}
    where in $(a)$ we used the triangle inequality for integrals \cite{rudin2021principles}, in $(b)$ we used the Cauchy Schwarz inequality for the $\ell_1$ norm $\left\vert \cdot\right\vert$, and in $(c)$ we used the assumption that $\Omega\left(\zeta_i,\zeta_j\right)$ is Lipshitz continuous with respect to $\zeta_i$. Therefore, if $\Omega\left(\zeta_i,\zeta_j\right)$ is Lipshitz continuous with respect to $\zeta_i$, then $\left(\Omega W\right)\left(\zeta_i,\zeta_j\right)$ is Lipshitz continuous with respect to $\zeta_i$ with a constant
    \begin{equation}
        L_i \triangleq L_{\Omega} \int \left\vert W\left(\zeta_j,\zeta_j\right)\right\vert d\zeta'\nonumber. 
    \end{equation}

    We now need to prove that $\left(\Omega W\right)\left(\zeta_i,\zeta_j\right)$ is Lipshitz continuous with respect to $\zeta_j$:
    \begin{align}
        & \left\vert \int \Omega\left(\zeta_i,\zeta'\right)W\left(\zeta',\zeta_j\right) - \Omega\left(\zeta_i,\zeta'\right)W\left(\zeta',\tilde{\zeta}_j\right)d\zeta'\right\vert \nonumber \\
        & ~~ = \left\vert \int \Omega\left(\zeta_i,\zeta'\right)\left(W\left(\zeta',\zeta_j\right) - W\left(\zeta',\tilde{\zeta}_j\right)\right)d\zeta'\right\vert \nonumber \\
        & ~~ \leq \int \left\vert \Omega\left(\zeta_i,\zeta'\right)\right\vert \left\vert W\left(\zeta',\zeta_j\right) - W\left(\zeta',\tilde{\zeta}_j\right)\right\vert d\zeta' \nonumber \\
        & ~~ \leq \int L_{W} \left\vert \Omega\left(\zeta_i,\zeta'\right)\right\vert \left\Vert \zeta_j - \tilde{\zeta}_j\right\Vert_2 d\zeta' \nonumber,
    \end{align}
    \begin{multline}
        \left\vert \int \Omega\left(\zeta_i,\zeta'\right)W\left(\zeta',\zeta_j\right) - \Omega\left(\zeta_i,\zeta'\right)W\left(\zeta',\tilde{\zeta}_j\right)d\zeta'\right\vert \\ \leq
         L_{W} \left\Vert \zeta_j - \tilde{\zeta}_j\right\Vert_2 \left\vert \int\Omega\left(\zeta_i,\zeta'\right)\right\vert  d\zeta' \nonumber.
    \end{multline}

    Based on above, if $W\left(\zeta_i,\zeta_j\right)$ is Lipshitz continuous with respect to $\zeta_j$, then $\left(\Omega W\right)\left(\zeta_i,\zeta_j\right)$ is Lipshitz continuous with respect to $\zeta_j$ with a constant
    \begin{equation}
        L_j \triangleq L_{W} \left\vert \int\Omega\left(\zeta_i,\zeta'\right)\right\vert  d\zeta' \nonumber. 
    \end{equation}

    We know that $\left(\Omega W\right)\left(\zeta_i,\zeta_j\right)$ is Lipshitz continuous with respect to $\zeta_i$ and $\Omega\left(\zeta_i,\zeta_j\right)$ is Lipshitz continuous with respect to $\zeta_j$. In addition, based on the derivations above, we know that the product of Lipshitz continuous functions is also Lipshitz continuous. Therefore, $\left(\Omega W \Omega'\right)\left(\zeta_i,\zeta_j\right)$ is Lipshitz continuous with respect to $\zeta_i$ and $\zeta_j$. Finally, if we assume that $\Psi_0\left(\zeta_i,\zeta_j\right)$ is Lipshitz continuous with respect to both $\zeta_i$ and $\zeta_j$, then this completes the proof that $\Psi_t\left(\zeta_i,\zeta_j\right)$ is Lipshitz continuous with respect to both $\zeta_i$ and $\zeta_j$. This also proves that any sampled function $f_t \sim \mathcal{G}\mathcal{P}\left(\mathbf{0},\Psi_t\right)$ is almost surely continuous. 

\end{proof}

\subsection{Proof for Lemma \ref{lemma:LGDS_properties}}

\begin{proof}
    Recall that the matrix $\Omega$ is contracting with a constant $C$. Therefore, the following inequality is true
    \begin{equation}\label{eq:inequality_omega_1}
        \left\Vert \Omega f_1 - \Omega f_2 \right\Vert_2 \leq C \left\Vert f_1 - f_2 \right\Vert_2 . 
    \end{equation}

    The norm on the left of inequality \eqref{eq:inequality_omega_1} is the following     \begin{multline}
        \left\Vert \Omega f_1 - \Omega f_2 \right\Vert_2 = \\ \sqrt{\iint \Omega\left(\zeta,\zeta'\right)^2\left(f_1\left(\zeta'\right)-f_2\left(\zeta'\right)\right)^2 d\zeta' d\zeta}\nonumber. 
    \end{multline}

    The following inequalities are true
    \begin{align}
        & \sqrt{\iint \Omega\left(\zeta,\zeta'\right)^2\left(f_1\left(\zeta'\right)-f_2\left(\zeta'\right)\right)^2 d\zeta' d\zeta} \nonumber \\
        & ~~~~~~~~~~~~ \geq \sqrt{\sum_{i=1}^d\sum_{j=1}^d \Gamma[i,j]^2 \left(f_1\left(\zeta_j\right)-f_2\left(\zeta_j\right)\right)^2} \nonumber \\
        & ~~~~~~~~~~~~ = \left\Vert \Gamma \left(z_1 - z_2\right)\right\Vert_2\nonumber, 
    \end{align}
    \begin{equation}
        \Rightarrow C \left\Vert f_1 - f_2 \right\Vert_2 > \left\Vert \Omega f_1 - \Omega f_2 \right\Vert_2 \geq \left\Vert \Gamma \left(z_1 - z_2\right)\right\Vert_2\nonumber.
    \end{equation}
    
    The above implies that the eigenvalues of $\Gamma$ are in the unit circle ($\Gamma$ is Schur). To prove that the matrix pair $\left(\Gamma, Q^{1/2}\right)$ is controllable, recall in Assumption \ref{assum:controllability} that $\Psi_t$ is positive definite. The definition of positive definiteness mentioned in Assumption \ref{assum:continuous_controllability} implies that the matrix $Z$, which satisfies the Lyapunov equation $Z=\Gamma Z\Gamma^\top + Q$, is positive definite. Therefore, $\left(\Gamma,Q^{1/2}\right)$ is controllable. 
    
\end{proof}

\subsection{Proof for Theorem \ref{theorem:code_seq}}

\begin{proof}
    Let there be the iteration of $P_{t|t-1}$ in \eqref{eq:DARE} where $g\left(P,\mathbf{e}_A\right)$ is defined in \eqref{eq:g_def}. Let $L_{A_t}$ be defined as \eqref{eq:modified_Kalman_filter}. According to \cite{1333199}, the matrix inequality is true 
    \begin{multline}\label{eq:matrix_inequality_1}
        g\left(P_{t|t-1},\mathbf{e}_{A_t}\right) \preceq \\ \left(\Gamma  - \Gamma L_{A_t} \mathbf{e}_{A_t}^\top\right)P_{t|t-1} \left(\Gamma  - \Gamma L_{A_t} \mathbf{e}_{A_t}^\top\right)^\top \\ + Q + \sigma^2 \Gamma L_{A_t} L_{A_t}^\top \Gamma^\top. 
    \end{multline}

    If $P_{\overline{A}} \succeq P_{t|t-1}$, then the matrix inequality in \eqref{eq:matrix_inequality_1} has the following upper bound 
    \begin{multline}
        g\left(P_{t|t-1},\mathbf{e}_{A_t}\right) \preceq \\ \left(\Gamma  - \Gamma L_{A_t} \mathbf{e}_{A_t}^\top\right)P_{\overline{A}} \left(\Gamma  - \Gamma L_{A_t} \mathbf{e}_{A_t}^\top\right)^\top \\ + Q + \sigma^2 \Gamma L_{A_t} L_{A_t}^\top \Gamma^\top \nonumber , 
    \end{multline}
    \begin{equation}
        \Rightarrow g\left(P_{t|t-1},\mathbf{e}_{A_t}\right) \preceq g\left(P_{\overline{A}},\mathbf{e}_{A_t}\right) \nonumber. 
    \end{equation}

    For the bound $g\left(P_{\overline{A}},\mathbf{e}_{A_t}\right)$, note the following details about the Kalman filter 
    \begin{enumerate}[label=(\Alph*)]
        \item If $P' \succeq P$, then $g\left(P',\mathbf{e}_{A_t}\right) \succeq g\left(P,\mathbf{e}_{A_t}\right)$. \label{item:monotonicity}
        \item If $\left(\Gamma,\mathbf{e}_{A_t}^\top\right)$ is detectable, then $P_{t|t-1}$ converges exponentially to $P_A$. \label{item:convergence}
    \end{enumerate}
    
    If we assume that $P_{\overline{A}} \preceq g\left(P_{\overline{A}},\mathbf{e}_{A_t}\right)$, then either \ref{item:monotonicity} or \ref{item:convergence} is false. Therefore, $P_{\overline{A}} \succeq g\left(P_{\overline{A}},\mathbf{e}_{A_t}\right)$ must be true for any $\mathbf{e}_{A_t}$. Next, since $g\left(P_{\overline{A}},\mathbf{e}_{A_t}\right)$ is bounded and $P_{t|t-1}$ converges exponentially, then $\Gamma  - \Gamma L_{A_t} \mathbf{e}_{A_t}^\top$ must be stable. This finalizes the proof. 
    
\end{proof}

\subsection{Proof of Theorem \ref{theorem:model_error}}\label{appendix:model_error}

For proving the bound on regret defined as \eqref{eq:regret_def} derived in Theorem \ref{theorem:regret}, we provide a theorem that bounds the model error of $\hat{G}_{a}^t\left(\mathbf{c}_i\right)$. First, before proving Theorem \ref{theorem:model_error}, the following lemmas are used. 

\begin{lemma}\label{lemma:bias_error}
    The norm $\left\Vert \mathbf{Z}_{\mathcal{T}_{a \mid \mathbf{c}_i}} \mathbf{B}_{\mathcal{T}_{a \mid \mathbf{c}_i}}^\top\right\Vert_{V_a^t\left(\mathbf{c}_i\right)^{-1}}$ has the following upper-bound which is satisfied with a probability of at least $1-\delta$:
    \begin{multline}\label{eq:model_error_term_1}
         \left\Vert \mathbf{Z}_{\mathcal{T}_{a \mid \mathbf{c}_i}} \mathbf{B}_{\mathcal{T}_{a \mid \mathbf{c}_i}}^\top\right\Vert_{V_a^t\left(\mathbf{c}_i\right)^{-1}} \leq \\
         \sqrt{2N_a B_z^2\log\left(n/\delta\right)\mbox{tr}\left(I_s - \lambda V_a^t\left(\mathbf{c}_i\right)^{-1} \right)} . 
    \end{multline}
    where $\mathbf{B}_{\mathcal{T}_{a \mid \mathbf{c}_i}}$ and $B_z \geq 0$ are defined as
    \begin{align}
        \mathbf{B}_{\mathcal{T}_{a \mid \mathbf{c}_i}} & \triangleq \begin{pmatrix}
            \beta_{A_{t_1}}^{t_1} & \dots & \beta_{A_{t_{N_a}}}^{t_{N_a}}
        \end{pmatrix} \in \mathbb{R}^{1 \times N_a} \label{eq:B_def}\\
        B_z^2 & \geq \mathbb{E}\left[\left\Vert\hat{z}_t'\right\Vert_2^2\right]\label{eq:b_z_def},         
    \end{align}
    and $\hat{z}_t'$ is based on the modified Kalman filter defined in \eqref{eq:modified_Kalman_filter}. 
\end{lemma}
\begin{proof}
    Using the Cauchy-Schwarz inequality and triangle inequality, we can bound $\left\Vert \mathbf{Z}_{\mathcal{T}_{a \mid \mathbf{c}_i}} \mathbf{B}_{\mathcal{T}_{a \mid \mathbf{c}_i}}^\top\right\Vert_{V_a^t\left(\mathbf{c}_i\right)^{-1}}$ as follows:
    \begin{equation}
        \left\Vert \mathbf{Z}_{\mathcal{T}_{a \mid \mathbf{c}_i}} \mathbf{B}_{\mathcal{T}_{a \mid \mathbf{c}_i}}^\top\right\Vert_{V_a^t\left(\mathbf{c}_i\right)^{-1}} = \left\Vert \sum_{\tau \in \mathcal{T}_{a\mid\mathbf{c}_i}} \Xi_{\tau}\left(\mathbf{c}_i\right)\beta_{A_{\tau}}^{\tau} \right\Vert_{V_a^t\left(\mathbf{c}_i\right)^{-1}} \nonumber ,
    \end{equation}
    \begin{multline}\label{eq:model_error_term_4}
        \Rightarrow \left\Vert \mathbf{Z}_{\mathcal{T}_{a \mid \mathbf{c}_i}} \mathbf{B}_{\mathcal{T}_{a \mid \mathbf{c}_i}}^\top\right\Vert_{V_a^t\left(\mathbf{c}_i\right)^{-1}} \leq\\
        \sum_{\tau \in \mathcal{T}_{a\mid\mathbf{c}_i}} \left\vert \beta_{A_{\tau}}^{\tau} \right\vert \left\Vert \Xi_{\tau}\left(\mathbf{c}_i\right)\right\Vert_{V_a^t\left(\mathbf{c}_i\right)^{-1}}. 
    \end{multline}
    
    For bounding $\beta_{A_{\tau}}^{\tau}$, recall that $\beta_{A_{\tau}}^{\tau}$ is defined as \eqref{eq:beta_def}. Using the Cauchy-Schwarz inequality we can bound $\beta_{A_{\tau}}^{\tau}$ as follows
    \begin{equation}
        \left\vert \beta_{A_{\tau}}^{\tau}\right\vert \leq \left\Vert \mathbf{e}_{A_{\tau}} \right\Vert_2 \left\Vert\left(\prod_{T=1}^s \left(\Gamma - \Gamma L_{A_{\tau-T}} \mathbf{e}_{A_{\tau-T}}^\top\right) \right) \hat{z}_{\tau-s}'\right\Vert_2 \nonumber, 
    \end{equation}
    \begin{equation}\label{eq:term_10}
        \Rightarrow \left\vert \beta_{A_{\tau}}^{\tau}\right\vert \leq \left\Vert\left(\prod_{T=1}^s \left(\Gamma - \Gamma L_{A_{\tau-T}} \mathbf{e}_{A_{\tau-T}}^\top\right) \right) \hat{z}_{\tau-s}'\right\Vert_2 . 
    \end{equation}

    Let there be a vector $v_{\hat{z}}$ where $\left\Vert v_{\hat{z}}\right\Vert_2 = 1$ and is defined as
    \begin{equation}
        v_{\hat{z}} \triangleq \underset{v_{\hat{z}}}{\arg\max} \max_{T \in [s]} \left\Vert \left(\Gamma - \Gamma L_{A_{\tau-T}} \mathbf{e}_{A_{\tau-T}}^\top\right) v_{\hat{z}}\right\Vert_2 \nonumber,
    \end{equation}    
    or is the eigenvector associated with the maximum magnitude eigenvalue. We can upper bound the right-side of \eqref{eq:term_10} using the following
    \begin{align}
        & \left\Vert\left(\prod_{T=1}^s \left(\Gamma - \Gamma L_{A_{\tau-T}} \mathbf{e}_{A_{\tau-T}}^\top\right) \right) \hat{z}_{\tau-s}'\right\Vert_2 \nonumber \\
        & ~~~~ = \left\Vert\left(\prod_{T=1}^s \left(\Gamma - \Gamma L_{A_{\tau-T}} \mathbf{e}_{A_{\tau-T}}^\top\right) \right) \frac{\hat{z}_{\tau-s}'}{\left\Vert \hat{z}_{\tau-s}'\right\Vert_2} \left\Vert \hat{z}_{\tau-s}'\right\Vert_2\right\Vert_2  \nonumber \\
        & ~~~~ \leq \left\Vert\left(\prod_{T=1}^s \left(\Gamma - \Gamma L_{A_{\tau-T}} \mathbf{e}_{A_{\tau-T}}^\top\right) \right) v_{\hat{z}} \right\Vert_2 \left\Vert \hat{z}_{\tau-s}'\right\Vert_2 \nonumber, 
    \end{align}    
    \begin{multline}
        \Rightarrow \left\Vert\left(\prod_{T=1}^s \left(\Gamma - \Gamma L_{A_{\tau-T}} \mathbf{e}_{A_{\tau-T}}^\top\right) \right) \hat{z}_{\tau-s}'\right\Vert_2  \\
        \leq \max_{T} \rho\left(\Gamma - \Gamma L_{A_{\tau-T}} \mathbf{e}_{A_{\tau-T}}^\top\right)^s \left\Vert \hat{z}_{\tau-s}'\right\Vert_2 \label{eq:term_11} . 
    \end{multline}

    For the upper bound \eqref{eq:term_11}, Theorem \ref{theorem:code_seq} has proven that $\Gamma - \Gamma L_{A_{\tau-T}}\mathbf{e}_{A_{\tau-T}}^\top$ is Schur. Therefore, the spectral radius to the $s$ power in \eqref{eq:term_11} has the following bound 
    \begin{equation}
        \max_{T \in [s]} \rho\left(\Gamma - \Gamma L_{A_{\tau-T}} \mathbf{e}_{A_{\tau-T}}^\top\right)^s \leq 1 \nonumber. 
    \end{equation}


    Since $\Gamma$ is Schur, then $z_{t-s}$ is bounded. Therefore, $\hat{z}_{t-s}'$ is also bounded. Given \eqref{eq:term_11}, we bound \eqref{eq:term_10} using Lemma 1 in \cite{gornet2024adaptivemethodnonstationarystochastic}, which states that the following inequalities are satisfied with a probability of at least $1-\delta$ for any $t \in [n]$
    \begin{equation}
        \left\vert \beta_{A_{\tau}}^{\tau}\right\vert \leq \left\Vert \hat{z}_{\tau-s}'\right\Vert_2 \leq \sqrt{2B_z^2\log\left(n/\delta\right)} \nonumber, 
    \end{equation}
    where $B_z$ is defined as \eqref{eq:b_z_def}. Therefore, using the inequality above, the sum \eqref{eq:model_error_term_4} can be bounded as follows
    \begin{multline}
        \left\Vert \mathbf{Z}_{\mathcal{T}_{a \mid \mathbf{c}_i}} \mathbf{B}_{\mathcal{T}_{a \mid \mathbf{c}_i}}^\top\right\Vert_{V_a^t\left(\mathbf{c}_i\right)^{-1}} \leq \\
        \sqrt{2B_z^2\log\left(n/\delta\right)} \sum_{\tau \in \mathcal{T}_{a\mid\mathbf{c}_i}}  \left\Vert \Xi_{\tau}\left(\mathbf{c}_i\right)\right\Vert_{V_a^t\left(\mathbf{c}_i\right)^{-1}}. 
    \end{multline}

    Next, we have to bound the sum $\sum_{\tau \in \mathcal{T}_{a\mid\mathbf{c}_i}}  \left\Vert \Xi_{\tau}\left(\mathbf{c}_i\right)\right\Vert_{V_a^t\left(\mathbf{c}_i\right)^{-1}}$
    \begin{align}
        & \sum_{\tau \in \mathcal{T}_{a\mid\mathbf{c}_i}}  \left\Vert \Xi_{\tau}\left(\mathbf{c}_i\right)\right\Vert_{V_a^t\left(\mathbf{c}_i\right)^{-1}} \nonumber \\
        &~~~~~~~~~ \leq \sqrt{N_a}\sqrt{\sum_{\tau \in \mathcal{T}_{a\mid\mathbf{c}_i}}  \left\Vert \Xi_{\tau}\left(\mathbf{c}_i\right)\right\Vert_{V_a^t\left(\mathbf{c}_i\right)^{-1}}^2} \nonumber \\
        &~~~~~~~~~ = \sqrt{N_a}\sqrt{\mbox{tr}\left(\sum_{\tau \in \mathcal{T}_{a\mid\mathbf{c}_i}} \Xi_{\tau}\left(\mathbf{c}_i\right)\Xi_{\tau}\left(\mathbf{c}_i\right)^\top V_a^t\left(\mathbf{c}_i\right)^{-1} \right)} \nonumber, 
    \end{align}
    \begin{multline}
        \sum_{\tau \in \mathcal{T}_{a\mid\mathbf{c}_i}}  \left\Vert \Xi_{\tau}\left(\mathbf{c}_i\right)\right\Vert_{V_a^t\left(\mathbf{c}_i\right)^{-1}} \leq  \sqrt{N_a}\sqrt{\mbox{tr}\left(I_s - \lambda V_a^t\left(\mathbf{c}_i\right)^{-1} \right)} \nonumber,
    \end{multline}
    leading to \eqref{eq:model_error_term_1}. 
\end{proof}

\begin{lemma}\label{lemma:alternating_bound}
    The norm $\left\Vert \mathbf{Z}_{\mathcal{T}_{a \mid \mathbf{c}_i}}\mathbf{P}_{\mathcal{T}_{a \mid \mathbf{c}_i}}^\top\right\Vert_{V_a^t\left(\mathbf{c}_i\right)^{-1}}$ has the following upper-bound which is satisfied with a probability of at least $1-\delta$:
    \begin{multline}
        \left\Vert \mathbf{Z}_{\mathcal{T}_{a \mid \mathbf{c}_i}}\mathbf{P}_{\mathcal{T}_{a \mid \mathbf{c}_i}}^\top\right\Vert_{V_a^t\left(\mathbf{c}_i\right)^{-1}} \leq \\ 2B_G B_{\Xi} \sqrt{N_a \mbox{tr}\left(I_s - \lambda V_a^t\left(\mathbf{c}_i\right)^{-1}\right)} \label{eq:P_bound}. 
    \end{multline}
    where $\mathbf{P}_{\mathcal{T}_{a \mid \mathbf{c}_i}}$, $B_G$, and $B_{\Xi}$ are defined as 
    \begin{align}
        \mathbf{P}_{\mathcal{T}_{a \mid \mathbf{c}_i}} & \triangleq \begin{pmatrix}
            \phi_{A_{t_1}}^{t_1} & \dots & \phi_{A_{t_{N_a}}}^{t_{N_a}}
        \end{pmatrix} \in \mathbb{R}^{1 \times N_a} \label{eq:P_def} \\
        B_G & \geq \left\Vert G_{a}\left(\mathbf{c}_i\right)\right\Vert_2 \mbox{ for all } \left(a,i\right) \in [d] \times [h], \nonumber \\
        & ~~~~~~~~~~~~~~~~~~~~~ \mathbf{c}_i\in \left(\zeta_{a_1},\dots,\zeta_{a_s}\right) \mbox{ for } \zeta_{A[i]}[i]\label{eq:b_g_definition} \\
        B_{\Xi} & \triangleq \sqrt{2\mathbb{E}\left[\left\Vert \Xi_t\left(\mathbf{c}_i\right) \right\Vert_2^2\right]\log\left(n/\delta\right)} \label{eq:B_Xi}. 
    \end{align}
\end{lemma}
\begin{proof}
    We can use the same strategy for bounding this term as we did for bounding $\left\Vert \mathbf{Z}_{\mathcal{T}_{a \mid \mathbf{c}_i}} \mathbf{B}_{\mathcal{T}_{a \mid \mathbf{c}_i}}^\top\right\Vert_{V_a^t\left(\mathbf{c}_i\right)^{-1}}$ in Lemma \ref{lemma:bias_error}. First, we use the Cauchy-Schwarz inequality and triangle inequality:
    \begin{equation}
        \left\Vert \mathbf{Z}_{\mathcal{T}_{a \mid \mathbf{c}_i}} \mathbf{P}_{\mathcal{T}_{a \mid \mathbf{c}_i}}^\top\right\Vert_{V_a^t\left(\mathbf{c}_i\right)^{-1}} = \left\Vert \sum_{\tau \in \mathcal{T}_{a\mid\mathbf{c}_i}} \Xi_{\tau}\left(\mathbf{c}_i\right)\phi_{A_{\tau}}^{\tau} \right\Vert_{V_a^t\left(\mathbf{c}_i\right)^{-1}} \nonumber ,
    \end{equation}
    \begin{multline}\label{eq:phi_error_term_4}
        \Rightarrow \left\Vert \mathbf{Z}_{\mathcal{T}_{a \mid \mathbf{c}_i}} \mathbf{P}_{\mathcal{T}_{a \mid \mathbf{c}_i}}^\top\right\Vert_{V_a^t\left(\mathbf{c}_i\right)^{-1}} \leq\\
        \sum_{\tau \in \mathcal{T}_{a\mid\mathbf{c}_i}} \left\vert \phi_{A_{\tau}}^{\tau} \right\vert \left\Vert \Xi_{\tau}\left(\mathbf{c}_i\right)\right\Vert_{V_a^t\left(\mathbf{c}_i\right)^{-1}}. 
    \end{multline}
    
    For bounding $\phi_{A_{\tau}}^{\tau}$, recall that $\phi_{A_{\tau}}^{\tau}$ is defined as \eqref{eq:phi_def}. Using the Cauchy-Schwarz inequality we can bound $\phi_{A_{\tau}}^{\tau}$ as follows
    \begin{equation}
        \left\vert \phi_{A_{\tau}}^{\tau}\right\vert \leq \left\vert \left(G_{a_{\tau}}\left(\mathbf{c}_i\right) - G_{a_{\tau}'}\left(\mathbf{c}_i\right) \right)^\top \Xi_{\tau}\left(\mathbf{c}_i\right) \right\vert \nonumber, 
    \end{equation}
    \begin{equation}
        \Rightarrow \left\vert \phi_{A_{\tau}}^{\tau}\right\vert \leq \left\Vert G_{a_{\tau}}\left(\mathbf{c}_i\right) - G_{a_{\tau}'}\left(\mathbf{c}_i\right) \right\Vert_2 \left\Vert \Xi_{\tau}\left(\mathbf{c}_i\right) \right\Vert_2  \label{eq:phi_bound_term_1}. 
    \end{equation}

    According to Lemma 1 of \cite{gornet2024adaptivemethodnonstationarystochastic}, we can upper-bound $\left\Vert \Xi_{\tau}\left(\mathbf{c}_i\right)\right\Vert_2$ with the following inequality which is satisfied with a probability of at least $1-\delta$:
    \begin{equation}\label{eq:xi_bound}
        \left\Vert \Xi_{\tau}\left(\mathbf{c}_i\right)\right\Vert_2 \leq B_{\Xi}, 
    \end{equation}
    where $B_{\Xi}$ is defined in \eqref{eq:B_Xi}. For $\left\Vert G_{a_{\tau}}\left(\mathbf{c}_i\right) - G_{a_{\tau}'}\left(\mathbf{c}_i\right) \right\Vert_2 $ in \eqref{eq:phi_bound_term_1}, we bound this term using the triangle inequality 
    \begin{align}
        \left\Vert G_{a_{\tau}}\left(\mathbf{c}_i\right) - G_{a_{\tau}'}\left(\mathbf{c}_i\right) \right\Vert_2 & \leq \left\Vert G_{a_{\tau}}\left(\mathbf{c}_i\right)\right\Vert_2 +\left\Vert G_{a_{\tau}'}\left(\mathbf{c}_i\right) \right\Vert_2 \nonumber \\
        & \leq 2B_G \nonumber, 
    \end{align}
    where we used the bound $\left\Vert G_{a_{\tau}}\left(\mathbf{c}_i\right)\right\Vert_2 \leq B_G$. Therefore, \eqref{eq:phi_bound_term_1} has the following bound 
    \begin{equation}
        \left\vert \phi_{A_{\tau}}^{\tau}\right\vert \leq 2B_G B_{\Xi} \nonumber. 
    \end{equation}

    For the sum $\sum_{{\tau} \in \mathcal{T}_{a\mid\mathbf{c}_i}}  \left\Vert \Xi_{\tau}\left(\mathbf{c}_i\right)\right\Vert_{V_a^t\left(\mathbf{c}_i\right)^{-1}}$, we use the same bound derived for bounding $\left\Vert \mathbf{Z}_{\mathcal{T}_{a \mid \mathbf{c}_i}} \mathbf{B}_{\mathcal{T}_{a \mid \mathbf{c}_i}}^\top\right\Vert_{V_a^t\left(\mathbf{c}_i\right)^{-1}}$. Therefore, $\left\Vert \mathbf{Z}_{\mathcal{T}_{a \mid \mathbf{c}_i}} \mathbf{P}_{\mathcal{T}_{a \mid \mathbf{c}_i}}^\top\right\Vert_{V_a^t\left(\mathbf{c}_i\right)^{-1}}$ has the following bound \eqref{eq:P_bound}. 
\end{proof}

Using Lemmas \ref{lemma:bias_error} and \ref{lemma:alternating_bound}, we can now prove Theorem \ref{theorem:model_error}. 

\begin{theorem}\label{theorem:model_error}
    Let $\hat{G}_{a}^t \left(\mathbf{c}_i\right)$ be identified based on \eqref{eq:learn_G} where the reward $\mathbf{X}_{\mathcal{T}_{a\mid\mathbf{c}_i}}$ has expression \eqref{eq:linear_model}. The following inequality is satisfied with a probability of at least $1-3\delta$ 
    \begin{equation}
        \left\Vert \hat{G}_{a}^t \left(\mathbf{c}_i\right) - G_{a} \left(\mathbf{c}_i\right) \right\Vert_{V_a^t\left(\mathbf{c}_i\right)} \leq b_{a \mid \mathbf{c}_i}\left(\delta\right) \nonumber, 
    \end{equation}
    where $b_{a\mid \mathbf{c}_i}\left(\delta\right)$ is defined as 
    \begin{multline}
        b_{a \mid \mathbf{c}_i}\left(\delta\right) \triangleq \sqrt{2B_{\varepsilon}^2\log\left(\frac{1}{\delta}\frac{ \det(V_a^t\left(\mathbf{c}_i\right))^{1/2}}{\det(\lambda I)^{1/2}}\right)}  \\ + 2B_G B_{\Xi} \sqrt{N_a \mbox{tr}\left(I_s - \lambda V_a^t\left(\mathbf{c}_i\right)^{-1}\right)} +  \lambda \sqrt{\mbox{tr}\left(V_a^t\left(\mathbf{c}_i\right)^{-1}\right)} B_G\\ +  \sqrt{2N_aB_z^2\log\left(n/\delta\right)\mbox{tr}\left(I-\lambda V_a^t\left(\mathbf{c}_i\right)^{-1}\right)}  \label{eq:probabilistic_bound},
    \end{multline}
    where $B_z$ is defined as \eqref{eq:b_z_def}, $B_G$ is defined as \eqref{eq:b_g_definition}, $B_{\Xi}$ is defined as \eqref{eq:B_Xi}, and $B_{\varepsilon}$ is defined as 
    \begin{equation}
        B_{\varepsilon} \geq \mbox{Var}\left(\gamma_A^t\right), ~ A \in [d]^h\label{eq:B_R}. 
    \end{equation}
    
    Therefore, based on above and with a probability of at least $1-3\delta$, the following inequality is true
    \begin{multline}\label{eq:prediction_bound}
        G_{a}\left(\mathbf{c}_i\right)^\top\Xi_t\left(\mathbf{c}_i\right) \leq \hat{G}_{a}^t\left(\mathbf{c}_i\right)^\top\Xi_t\left(\mathbf{c}_i\right) \\ + b_{a \mid \mathbf{c}_i}\left(\delta\right)\left\Vert \Xi_t\left(\mathbf{c}_i\right)\right\Vert_{V_a^t\left(\mathbf{c}_i\right)^{-1}}. 
    \end{multline}
\end{theorem}

\begin{proof}
    Let there be the following linear model 
    \begin{equation}
        \mathbf{X}_{\mathcal{T}_{a \mid \mathbf{c}_i}} = G_{a}\left(\mathbf{c}_i\right)^\top\mathbf{Z}_{\mathcal{T}_{a \mid \mathbf{c}_i}} + \mathbf{B}_{\mathcal{T}_{a \mid \mathbf{c}_i}} + \mathbf{P}_{\mathcal{T}_{a \mid \mathbf{c}_i}} + \mathbf{E}_{\mathcal{T}_{a \mid \mathbf{c}_i}}\nonumber, 
    \end{equation}
    where $\mathbf{B}_{\mathcal{T}_{a \mid \mathbf{c}_i}}$ is defined as \eqref{eq:B_def}, $\mathbf{P}_{\mathcal{T}_{a \mid \mathbf{c}_i}}$ is defined as \eqref{eq:P_def}, and $\mathbf{E}_{\mathcal{T}_{a \mid \mathbf{c}_i}}$ is defined as
    \begin{equation}
        \mathbf{E}_{\mathcal{T}_{a \mid \mathbf{c}_i}} \triangleq \begin{pmatrix}
            \gamma_{A_{t_1}}^{t_1} & \dots & \gamma_{A_{t_{N_a}}}^{t_{N_a}}
        \end{pmatrix} \in \mathbb{R}^{1 \times N_a} \label{eq:E_def}.
    \end{equation}

    Multiplying by the pseudo-inverse $\mathbf{Z}_{\mathcal{T}_{a \mid \mathbf{c}_i}}^\top V_a^t\left(\mathbf{c}_i\right)^{-1}$ gives the following expression of $\hat{G}_{a}^t\left(\mathbf{c}_i\right)$
    \begin{multline}\label{eq:norm_term_1}
        \hat{G}_{a}^t\left(\mathbf{c}_i\right)^\top = G_{a}\left(\mathbf{c}_i\right)^\top\mathbf{Z}_{\mathcal{T}_{a \mid \mathbf{c}_i}}\mathbf{Z}_{\mathcal{T}_{a \mid \mathbf{c}_i}}^\top V_a^t\left(\mathbf{c}_i\right)^{-1} \\ + \mathbf{B}_{\mathcal{T}_{a \mid \mathbf{c}_i}}\mathbf{Z}_{\mathcal{T}_{a \mid \mathbf{c}_i}}^\top V_a^t\left(\mathbf{c}_i\right)^{-1} 
        + \mathbf{P}_{\mathcal{T}_{a \mid \mathbf{c}_i}}\mathbf{Z}_{\mathcal{T}_{a \mid \mathbf{c}_i}}^\top V_a^t\left(\mathbf{c}_i\right)^{-1}
         \\+ \mathbf{E}_{\mathcal{T}_{a \mid \mathbf{c}_i}}\mathbf{Z}_{\mathcal{T}_{a \mid \mathbf{c}_i}}^\top V_a^t\left(\mathbf{c}_i\right)^{-1}. 
    \end{multline}

    If we add the term $\lambda G_{a}\left(\mathbf{c}_i\right)^\top V_a^t\left(\mathbf{c}_i\right)^{-1} - \lambda G_{a}\left(\mathbf{c}_i\right)^\top V_a^t\left(\mathbf{c}_i\right)^{-1}$ we can simplify the terms on the right of \eqref{eq:norm_term_1} using the definition of $V_a^t\left(\mathbf{c}_i\right)$ in \eqref{eq:V_def}:
    \begin{multline}
        \hat{G}_{a}^t\left(\mathbf{c}_i\right)^\top = G_{a}\left(\mathbf{c}_i\right)^\top\left(\lambda I_s + \mathbf{Z}_{\mathcal{T}_{a \mid \mathbf{c}_i}}\mathbf{Z}_{\mathcal{T}_{a \mid \mathbf{c}_i}}^\top\right) V_a^t\left(\mathbf{c}_i\right)^{-1} \\ + \mathbf{B}_{\mathcal{T}_{a \mid \mathbf{c}_i}}\mathbf{Z}_{\mathcal{T}_{a \mid \mathbf{c}_i}}^\top V_a^t\left(\mathbf{c}_i\right)^{-1} + 
        \mathbf{P}_{\mathcal{T}_{a \mid \mathbf{c}_i}}\mathbf{Z}_{\mathcal{T}_{a \mid \mathbf{c}_i}}^\top V_a^t\left(\mathbf{c}_i\right)^{-1}
         \\+ \mathbf{E}_{\mathcal{T}_{a \mid \mathbf{c}_i}}\mathbf{Z}_{\mathcal{T}_{a \mid \mathbf{c}_i}}^\top V_a^t\left(\mathbf{c}_i\right)^{-1} - \lambda G_{a}\left(\mathbf{c}_i\right)^\top V_a^t\left(\mathbf{c}_i\right)^{-1}  \nonumber,
    \end{multline}
    \begin{multline}\label{eq:norm_term_2}
        \Rightarrow \hat{G}_{a}^t\left(\mathbf{c}_i\right)^\top - G_{a}\left(\mathbf{c}_i\right)^\top =  \mathbf{B}_{\mathcal{T}_{a \mid \mathbf{c}_i}}\mathbf{Z}_{\mathcal{T}_{a \mid \mathbf{c}_i}}^\top V_a^t\left(\mathbf{c}_i\right)^{-1} \\ + \mathbf{P}_{\mathcal{T}_{a \mid \mathbf{c}_i}}\mathbf{Z}_{\mathcal{T}_{a \mid \mathbf{c}_i}}^\top V_a^t\left(\mathbf{c}_i\right)^{-1} + \mathbf{E}_{\mathcal{T}_{a \mid \mathbf{c}_i}}\mathbf{Z}_{\mathcal{T}_{a \mid \mathbf{c}_i}}^\top V_a^t\left(\mathbf{c}_i\right)^{-1} \\ - \lambda G_{a}\left(\mathbf{c}_i\right)^\top V_a^t\left(\mathbf{c}_i\right)^{-1}.  
    \end{multline}

    Taking the weighted norm with weights $V_a^t\left(\mathbf{c}_i\right)$ and the triangle inequality provides the following inequality for \eqref{eq:norm_term_2}: 
    \begin{multline}
        \left\Vert \hat{G}_{a}^t\left(\mathbf{c}_i\right) - G_{a}\left(\mathbf{c}_i\right) \right\Vert_{V_a^t\left(\mathbf{c}_i\right)} = \\
        \left\Vert \mathbf{Z}_{\mathcal{T}_{a \mid \mathbf{c}_i}}\left( \mathbf{B}_{\mathcal{T}_{a \mid \mathbf{c}_i}} + \mathbf{P}_{\mathcal{T}_{a \mid \mathbf{c}_i}} + \mathbf{E}_{\mathcal{T}_{a \mid \mathbf{c}_i}}\right)^\top + \lambda G_{a}\left(\mathbf{c}_i\right) \right\Vert_{V_a^t\left(\mathbf{c}_i\right)^{-1}} \nonumber,
    \end{multline}
    \begin{multline}\label{eq:norm_term_3}
        \Rightarrow \left\Vert \hat{G}_{a}^t\left(\mathbf{c}_i\right) - G_{a}\left(\mathbf{c}_i\right) \right\Vert_{V_a^t\left(\mathbf{c}_i\right)} \leq  \left\Vert \mathbf{Z}_{\mathcal{T}_{a \mid \mathbf{c}_i}} \mathbf{B}_{\mathcal{T}_{a \mid \mathbf{c}_i}}^\top\right\Vert_{V_a^t\left(\mathbf{c}_i\right)^{-1}} \\ + \left\Vert \mathbf{Z}_{\mathcal{T}_{a \mid \mathbf{c}_i}}\mathbf{P}_{\mathcal{T}_{a \mid \mathbf{c}_i}}^\top\right\Vert_{V_a^t\left(\mathbf{c}_i\right)^{-1}} + \left\Vert \mathbf{Z}_{\mathcal{T}_{a \mid \mathbf{c}_i}}\mathbf{E}_{\mathcal{T}_{a \mid \mathbf{c}_i}}^\top\right\Vert_{V_a^t\left(\mathbf{c}_i\right)^{-1}} \\ + \left\Vert\lambda G_{a}\left(\mathbf{c}_i\right) \right\Vert_{V_a^t\left(\mathbf{c}_i\right)^{-1}}. 
    \end{multline}

    For bounding the terms in \eqref{eq:norm_term_3}, we use the following steps:

    \textbf{Bounding} $\left\Vert \mathbf{Z}_{\mathcal{T}_{a \mid \mathbf{c}_i}} \mathbf{B}_{\mathcal{T}_{a \mid \mathbf{c}_i}}^\top\right\Vert_{V_a^t\left(\mathbf{c}_i\right)^{-1}}$: Using Lemma \ref{lemma:bias_error}, the inequality \eqref{eq:model_error_term_1} is satisfied with a probability of at least $1-\delta$. 

    \textbf{Bounding} $\left\Vert \mathbf{Z}_{\mathcal{T}_{a \mid \mathbf{c}_i}} \mathbf{P}_{\mathcal{T}_{a \mid \mathbf{c}_i}}^\top\right\Vert_{V_a^t\left(\mathbf{c}_i\right)^{-1}}$: Using Lemma \ref{lemma:alternating_bound}, the inequality \eqref{eq:P_bound} is satisfied with a probability of at least $1-\delta$. 
    
    \textbf{Bounding} $\left\Vert \mathbf{Z}_{\mathcal{T}_{a \mid \mathbf{c}_i}} \mathbf{E}_{\mathcal{T}_{a \mid \mathbf{c}_i}}^\top\right\Vert_{V_a^t\left(\mathbf{c}_i\right)^{-1}}$: Using Theorem 20.4 in \cite{lattimore2020bandit}, the following inequality is satisfied with a probability of at least $1-\delta$
    \begin{multline}\label{eq:model_error_term_2}
        \left\Vert \mathbf{Z}_{\mathcal{T}_{a \mid \mathbf{c}_i}} \mathbf{E}_{\mathcal{T}_{a \mid \mathbf{c}_i}}^\top\right\Vert_{V_a^t\left(\mathbf{c}_i\right)^{-1}} \leq \\ \sqrt{2B_{\varepsilon}^2 \log\left(\frac{1}{\delta}\frac{\det\left(V_a^t\left(\mathbf{c}_i\right)\right)^{1/2}}{\lambda^{s/2}}\right)}.
    \end{multline}

    \textbf{Bounding} $\left\Vert\lambda G_{a}\left(\mathbf{c}_i\right) \right\Vert_{V_a^t\left(\mathbf{c}_i\right)^{-1}}$: Using the definition of $B_G$ in \eqref{eq:b_g_definition}, the following bound is satisfied 
    \begin{equation}\label{eq:model_error_term_3}
        \left\Vert\lambda G_{a}\left(\mathbf{c}_i\right) \right\Vert_{V_a^t\left(\mathbf{c}_i\right)^{-1}} \leq \lambda B_G \sqrt{\mbox{tr}\left(V_a^t\left(\mathbf{c}_i\right)^{-1}\right)}.
    \end{equation}

    Combining terms \eqref{eq:model_error_term_1}, \eqref{eq:P_bound}, \eqref{eq:model_error_term_2}, and \eqref{eq:model_error_term_3} gives the following bound for $\left\Vert \hat{G}_{a}^t\left(\mathbf{c}_i\right) - G_{a}\left(\mathbf{c}_i\right) \right\Vert_{V_a^t\left(\mathbf{c}_i\right)}$ which is satisfied with a probability of at least $1-3\delta$
    \begin{multline}
        \left\Vert \hat{G}_{a}^t\left(\mathbf{c}_i\right) - G_{a}\left(\mathbf{c}_i\right) \right\Vert_{V_a^t\left(\mathbf{c}_i\right)} \leq \lambda B_G \sqrt{\mbox{tr}\left(V_a^t\left(\mathbf{c}_i\right)^{-1}\right)} \\
        + \sqrt{2N_a B_z^2 \log\left(n/\delta\right)\mbox{tr}\left(I_s - \lambda V_a^t\left(\mathbf{c}_i\right)^{-1} \right)}\\+ 2B_G B_{\Xi} \sqrt{N_a \mbox{tr}\left(I_s - \lambda V_a^t\left(\mathbf{c}_i\right)^{-1}\right)}\\ +
        \sqrt{2B_{\varepsilon}^2 \log\left(\frac{1}{\delta}\frac{\det\left(V_a^t\left(\mathbf{c}_i\right)\right)^{1/2}}{\lambda^{s/2}}\right)} \nonumber. 
    \end{multline}    

    Finally, we can bound the prediction error
    \begin{equation}
        \left\vert \hat{G}_{a}^t\left(\mathbf{c}_i\right)^\top\Xi_t\left(\mathbf{c}_i\right) - G_{a}\left(\mathbf{c}_i\right)^\top\Xi_t\left(\mathbf{c}_i\right)\right\vert \nonumber.
    \end{equation}
    
    \textbf{Bounding} $\left\vert \hat{G}_{a}^t\left(\mathbf{c}_i\right)^\top\Xi_t\left(\mathbf{c}_i\right) - G_{a}\left(\mathbf{c}_i\right)^\top\Xi_t\left(\mathbf{c}_i\right)\right\vert$: By considering the following optimization problem 
    \begin{equation}
        \begin{array}{cc}
            \max_{\tilde{G}} & \tilde{G}^\top\Xi_t\left(\mathbf{c}_i\right) \\
            \mbox{s.t.} & \left\Vert \hat{G}_{a}^t\left(\mathbf{c}_i\right) - \tilde{G}\right\Vert_{V_a^t\left(\mathbf{c}_i\right)}^2 \leq b_{a \mid \mathbf{c}_i}\left(\delta\right)^2
        \end{array} \label{eq:prediction_optimization_problem}. 
    \end{equation}

    We can solve \eqref{eq:prediction_optimization_problem} using the Lagrangian 
    \begin{multline}\label{eq:lagrangian}
        L\left(\tilde{G},\epsilon\right) = \tilde{G}^\top\Xi_t\left(\mathbf{c}_i\right) \\ - \epsilon\left(\left\Vert \hat{G}_{a}^t\left(\mathbf{c}_i\right) - \tilde{G}\right\Vert_{V_a^t\left(\mathbf{c}_i\right)}^2 - b_{a \mid \mathbf{c}_i}\left(\delta\right)^2\right).
    \end{multline}

    The Lagrangian \eqref{eq:lagrangian} has the following partial derivatives 
    \begin{align}
        \frac{\partial L}{\partial \tilde{G}} & = \Xi_t\left(\mathbf{c}_i\right)^\top + 2\epsilon\hat{G}_{a}^t\left(\mathbf{c}_i\right)^\top V_a^t\left(\mathbf{c}_i\right) - 2\epsilon\tilde{G}^\top V_a^t\left(\mathbf{c}_i\right) \label{eq:partial_derivative_tilde_G} \\
        \frac{\partial L}{\partial \epsilon} & = \left\Vert \hat{G}_{a}^t\left(\mathbf{c}_i\right) - \tilde{G}\right\Vert_{V_a^t\left(\mathbf{c}_i\right)}^2 - b_{a \mid \mathbf{c}_i}\left(\delta\right)^2\label{eq:partial_derivative_epsilon}.
    \end{align}

    If we set the partial derivatives \eqref{eq:partial_derivative_tilde_G} and \eqref{eq:partial_derivative_epsilon} to $0$ and solve, we have the following solutions for $\tilde{G}$ and $\epsilon$:
    \begin{align}
        \tilde{G}^\top & = \hat{G}_{a}^t\left(\mathbf{c}_i\right)^\top + \frac{\Xi_t\left(\mathbf{c}_i\right)^\top  V_a^t\left(\mathbf{c}_i\right)^{-1}}{2\epsilon} \label{eq:partial_derivative_tilde_G_solution} \\
        \epsilon & = \frac{\sqrt{\Xi_t\left(\mathbf{c}_i\right)^\top  V_a^t\left(\mathbf{c}_i\right)^{-1} \Xi_t\left(\mathbf{c}_i\right)}}{2b_{a \mid \mathbf{c}_i}\left(\delta\right)}\label{eq:partial_derivative_epsilon_solution}.
    \end{align}

    Therefore, using \eqref{eq:partial_derivative_tilde_G_solution} and \eqref{eq:partial_derivative_epsilon_solution} provides the solution to the optimization problem 
    \begin{equation}
        \hat{G}_{a}^t\left(\mathbf{c}_i\right)^\top\Xi_t\left(\mathbf{c}_i\right) + b_{a \mid \mathbf{c}_i}\left(\delta\right)\left\Vert \Xi_t\left(\mathbf{c}_i\right)\right\Vert_{V_a^t\left(\mathbf{c}_i\right)^{-1}} \nonumber, 
    \end{equation}
    implying that \eqref{eq:prediction_bound} is satisfied with a probability of at least $1-3\delta$.

\end{proof}

\subsection{Proof of Theorem \ref{theorem:regret}}\label{appendix:regret}

In the main text, we did not add the bound on regret \eqref{eq:regret_def}; only the regret rate was added. The rationale for this omission is that regret rate provides better intuition on how HyperController performs. In the theorem below, we also add the regret bound. 

\begin{theorem}\label{theorem:extended_regret}
    Let the reward $X_t$ be the output of the LGDS \eqref{eq:LGDS} based on equation \eqref{eq:LGDS_output}. Let $G_{a}\left(\mathbf{c}_i\right)$ follow the linear model posed in \eqref{eq:linear_model} and be identified as $\hat{G}_{a}^t\left(\mathbf{c}_i\right)$ in \eqref{eq:learn_G}. If actions are selected based on optimization problem \eqref{eq:action_selection}, then regret $R_n$ defined in \eqref{eq:regret_def} satisfies the inequality with a probability of at least $1-13\delta$
    \begin{multline}\label{eq:regret_upper_bound}
        R_n \leq n \tilde{\Delta} + n \sqrt{4L\Delta\log\left(n/\delta\right)} \\ + hd^{s+1} \overline{B}_{\sigma} \max\left\{1,\frac{B_{\Xi}}{\sqrt{\lambda}}\right\} \sqrt{2sn\log\left(\frac{s\lambda + n B_{\Xi}}{s\lambda}\right)} \\
        h d^s  n \frac{\overline{B}_{\sigma} B_{\Xi}}{\sqrt{\lambda}}
        + 2n \sqrt{2B_z^2\log\left(n/\delta\right)} \\+ 4n B_G B_{\Xi} + \sqrt{2nB_{\varepsilon}^2\log\left(1/\delta\right)} ,  
    \end{multline}
    where $\tilde{\Delta} \geq 0$ is an arbitrary constant, $L$ is a Lipshitz constant of the function $\Psi_t$ defined in \eqref{eq:psi_covariance}, $\Delta \geq 0$ is the maximum distance between two neighboring values $\zeta_{a}, \zeta_{a'}$ in $\zeta_{A[i]}[i]$ for all $i \in [h]$. The bounds $B_z$ is defined as \eqref{eq:b_z_def}, $B_G$ is defined as \eqref{eq:b_g_definition}, $B_{\Xi}$ is defined as \eqref{eq:B_Xi}, $B_{\varepsilon}$ is defined as \eqref{eq:B_R}, and $\overline{B}_{\sigma}$ is defined as 
    \begin{multline}
        \overline{B}_{\sigma} \triangleq
        \sqrt{2B_{\varepsilon}^2\log\left(\frac{1}{\delta}\left(\frac{s\lambda + n B_{\Xi}}{s\lambda}\right)^{s/2}\right)} \\ + \sqrt{2s n B_z^2\log\left(n/\delta\right)} + 2B_G  B_{\Xi} \sqrt{sn} + \frac{\lambda B_G}{\sqrt{s}}
         \label{eq:B_sigma_def}.
    \end{multline}
    
    Therefore, regret has the asymptotic rate \eqref{eq:regret_rate}. 
\end{theorem}

\begin{proof}
    The instantaneous regret for one round $t$ is $r_t \triangleq X_t^* - X_t$, where $X_t^*$ is the highest possible reward that can be sampled at round $t$ from choosing $A_t^* \in [d]^h$ and $X_t$ is the reward sampled if action $A_t \in [d]^h$ is chosen based on optimization problem \eqref{eq:action_selection}. The instantaneous regret has the following expression 
    \begin{align}
        r_t & = X_t^* - X_t \nonumber \\
        & = f_t\left(\zeta^*\right) - f_t\left(\zeta_{A_t}\right) \nonumber,
        \label{eq:continuous_inequality}
    \end{align}

    Since we discretized $f_t$, we want to find the value $\zeta_{A_t^*}$ that is the closest to $\zeta^*$. We use the following inequalities
    \begin{multline}
        \left\vert f_t\left(\zeta^*\right) - f_t\left(\zeta_{A_t^*}\right) \right\vert \leq \\ \sqrt{2\mathbb{E}\left[\left(f_t\left(\zeta^*\right) - f_t\left(\zeta_{A_t^*}\right)\right)^2\right]\log\left(1/\delta\right)} \nonumber, 
    \end{multline}
    \begin{align}
        & \mathbb{E}\left[\left(f_t\left(\zeta^*\right) - f_t\left(\zeta_{A_t^*}\right)\right)^2\right]\nonumber\\
        & ~~~~~~ = \mathbb{E}\left[f_t\left(\zeta^*\right)^2 - 2f_t\left(\zeta^*\right)f_t\left(\zeta_{A_t^*}\right) + f_t\left(\zeta_{A_t^*}\right)^2\right] \nonumber \\
        & ~~~~~~ \leq \left(\Psi_t\left(\zeta^*,\zeta^*\right) - \Psi_t\left(\zeta^*,\zeta_{A_t^*}\right)\right) \nonumber \\
        & ~~~~~~~~~~~~~~~~~~ + \left(\Psi_t\left(\zeta_{A_t^*},\zeta_{A_t^*}\right)- \Psi_t\left(\zeta_{A_t^*},\zeta^*\right)\right) \nonumber \\
        & ~~~~~~ \leq 2L\left\Vert \zeta^* - \zeta_{A_t^*}\right\Vert_2 \nonumber \\
        & ~~~~~~ \leq 2L\Delta \nonumber,
    \end{align}
    
        
        
    \begin{equation}
        \Rightarrow \left\vert f_t\left(\zeta^*\right) - f_t\left(\zeta_{A_t^*}\right) \right\vert \leq \sqrt{4L\Delta\log\left(n/\delta\right)}  \nonumber,
    \end{equation}
    which is satisfied with a probability of at least $1-\delta$. Therefore, instantaneous regret can be bounded as
    \begin{equation}
        r_t \leq \sqrt{4L\Delta\log\left(n/\delta\right)}  + \left\langle \mathbf{e}_{A_t^*}, z_t \right\rangle - \left\langle \mathbf{e}_{A_t}, z_t \right\rangle\nonumber. 
    \end{equation}

    For bounding the difference $\left\langle \mathbf{e}_{A_t^*}, z_t \right\rangle  - \left\langle \mathbf{e}_{A_t}, z_t \right\rangle$, recall that HyperController selects $\zeta_{A[i]}[i]=\zeta_a$, $i \in [h]$ without regard to the values $\zeta_{A[j]}[j]=\zeta_{a'}$, $j \neq i$. Therefore, we are not guaranteed to converge to the optimal action. However, since we cannot guarantee that $f_t\left(\zeta\right)$ is convex/concave with respect to $\zeta$, then this implementation is worthwhile. Assume that $\mathbf{e}_{\tilde{A}_t^*}$ is the corresponding action where $\psi\left(\dots,\zeta_{\tilde{A}_t^*[i]}[i],\dots;\theta\left(t\right)\right)$ is maximized with respect to each component $i\in [h]$ (not the global maximum). Therefore, we have a linear increase in regret added
    \begin{equation}
        \left\langle \mathbf{e}_{A_t^*}, z_t \right\rangle - \left\langle \mathbf{e}_{A_t}, z_t \right\rangle \leq \tilde{\Delta}+ \left\langle \mathbf{e}_{\tilde{A}_t^*}, z_t \right\rangle - \left\langle \mathbf{e}_{A_t}, z_t \right\rangle \nonumber, 
    \end{equation}
    where $\tilde{\Delta} > 0$ is set such that $\left\langle \mathbf{e}_{A_t^*}, z_t \right\rangle - \left\langle \mathbf{e}_{\tilde{A}_t^*}, z_t \right\rangle \leq \tilde{\Delta}$. Next, we can express $\left\langle \mathbf{e}_{\tilde{A}_t^*}, z_t \right\rangle - \left\langle \mathbf{e}_{A_t}, z_t \right\rangle$ with $G_{a}\left(\mathbf{c}_i\right)$
    \begin{multline}
        \left\langle \mathbf{e}_{\tilde{A}_t^*}, z_t \right\rangle - \left\langle \mathbf{e}_{A_t}, z_t \right\rangle \\ \leq \sum_{i \in [h]} G_{\tilde{a}_t^*}\left(\mathbf{c}_i\right)^\top \Xi_t\left(\mathbf{c}_i\right)- G_{a_t}\left(\mathbf{c}_i\right)^\top \Xi_t\left(\mathbf{c}_i\right) \\ + \phi_{\tilde{A}_t^*}^t - \phi_{A_t}^t 
        +  \beta_{\tilde{A}_t^*}^t - \beta_{A_t}^t  +  \gamma_{\tilde{A}_t^*}^t - \gamma_{A_t}^t \nonumber. 
    \end{multline}

    According to Theorem \ref{theorem:model_error}, the following inequality is satisfied with a probability of at least $1-3\delta $
    \begin{multline}
        G_{\tilde{a}_t^*}\left(\mathbf{c}_i\right)^\top \Xi_t\left(\mathbf{c}_i\right) \leq \hat{G}_{\tilde{a}_t^*}^t \left(\mathbf{c}_i\right)^\top \Xi_t\left(\mathbf{c}_i\right) \\ + b_{\tilde{a}_t^* \mid \mathbf{c}_i}\left(\delta\right)\left\Vert \Xi_t\left(\mathbf{c}_i\right)\right\Vert_{V_{\tilde{a}_t^*}^t\left(\mathbf{c}_i\right)^{-1}}\label{eq:regret_inequality_1}. 
    \end{multline}

    Since action $\mathbf{e}_a$ is chosen at round $t$, then that implies the following inequality is true
    \begin{equation}
        \hat{G}_{\tilde{a}_t^*}^t \left(\mathbf{c}_i\right)^\top \Xi_t\left(\mathbf{c}_i\right) \leq
        \hat{G}_{a_t}^t \left(\mathbf{c}_i\right)^\top \Xi_t\left(\mathbf{c}_i\right)\label{eq:regret_inequality_2}. 
    \end{equation}

    Therefore, we can bound instantaneous regret $r_t$ using \eqref{eq:regret_inequality_2}:
    \begin{multline}
        r_t \leq \tilde{\Delta} + \sqrt{4L\Delta\log\left(n/\delta\right)}  + \sum_{i \in [h]} \hat{G}_{a_t}^t \left(\mathbf{c}_i\right)^\top \Xi_t\left(\mathbf{c}_i\right)\\ + b_{\tilde{a}_t^* \mid \mathbf{c}_i}\left(\delta\right)\left\Vert \Xi_t\left(\mathbf{c}_i\right)\right\Vert_{V_{\tilde{a}_t^*}^t\left(\mathbf{c}_i\right)^{-1}} - G_{a_t}\left(\mathbf{c}_i\right)^\top \Xi_t\left(\mathbf{c}_i\right) \\ + \beta_{\tilde{A}_t^*}^t - \beta_{A_t}^t  + \gamma_{\tilde{A}_t^*}^t - \gamma_{A_t}^t \nonumber. 
    \end{multline}
    \begin{multline}
        \Rightarrow r_t \overset{(a)}{\leq} \tilde{\Delta} + \sqrt{4L\Delta\log\left(n/\delta\right)} \\ + \sum_{i \in [h]} \sum_{a \in \left\{\tilde{a}_t^*,a_t\right\}} b_{a \mid \mathbf{c}_i}\left(\delta\right)\left\Vert \Xi_t\left(\mathbf{c}_i\right)\right\Vert_{V_{a}^t\left(\mathbf{c}_i\right)^{-1}} \\ + \beta_{\tilde{A}_t^*}^t - \beta_{A_t}^t  + \gamma_{\tilde{A}_t^*}^t - \gamma_{A_t}^t \nonumber. 
    \end{multline}
    where in $(a)$ we used inequality \eqref{eq:regret_inequality_1}. For bounding the term $b_{a \mid \mathbf{c}_i}\left(\delta\right)\left\Vert \Xi_t\left(\mathbf{c}_i\right)\right\Vert_{V_{a}^t\left(\mathbf{c}_i\right)^{-1}}$, $a = \tilde{a}_t^*, a_t$, we break it down to a number of steps. 
    
    \textbf{Bounding} $\Xi_t\left(\mathbf{c}_i\right)$: Since $X_t$ is the output of the LGDS \eqref{eq:LGDS}, then $X_t$ is a Gaussian random variable. Therefore, according to Lemma 1 of \cite{gornet2024adaptivemethodnonstationarystochastic}, the random variable $\left\Vert \Xi_t\left(\mathbf{c}_i\right) \right\Vert_2$ satisfies inequality \eqref{eq:xi_bound} with a probability of at least $1-\delta$ for any $t \in [n]$. 
    
    \textbf{Bounding} $b_{a\mid \mathbf{c}_i}\left(\delta\right)$, $a = \tilde{a}_t^*, a_t$: For the first term in $b_{a\mid \mathbf{c}_i}\left(\delta\right)$, according to Lemma 19.4 in \cite{lattimore2020bandit}, the following inequality is satisfied:
    \begin{multline}
        \sqrt{2B_{\varepsilon}^2\log\left(\frac{1}{\delta}\frac{ \det(V_a^t\left(\mathbf{c}_i\right))^{1/2}}{\det(\lambda I)^{1/2}}\right)} \\ \leq 
        \sqrt{2B_{\varepsilon}^2\log\left(\frac{1}{\delta}\left(\frac{s\lambda + N_a B_{\Xi}}{s\lambda}\right)^{s/2}\right)} \label{eq:term_2}. 
    \end{multline}

    For the second term in $b_{a\mid \mathbf{c}_i}\left(\delta\right)$, this has the following upper-bound
    \begin{multline}
        \sqrt{2N_a B_z^2 \log\left(N_a/\delta\right)\mbox{tr}\left(I-\lambda V_a^t\left(\mathbf{c}_i\right)^{-1}\right)} \\
        \leq \sqrt{2sN_a B_z^2\log\left(N_a/\delta\right)} \label{eq:term_3}. 
    \end{multline}

    For the third term in $b_{a\mid \mathbf{c}_i}\left(\delta\right)$, we use the same idea as bounding the first term: 
    \begin{equation}
        2 B_G B_{\Xi}\sqrt{N_a\mbox{tr}\left(I-\lambda V_a^t\left(\mathbf{c}_i\right)^{-1}\right)}
        \leq 2B_G B_{\Xi}\sqrt{sN_a} \label{eq:phi_term_1}. 
    \end{equation}
    
    For the fourth term in $b_{a\mid \mathbf{c}_i}\left(\delta\right)$, since $V_a^t\left(\mathbf{c}_i\right) \succeq \lambda I$ according to \eqref{eq:V_def}, then the following bound exists
    \begin{equation}
        \lambda \sqrt{\mbox{tr}\left(V_a^t\left(\mathbf{c}_i\right)^{-1} \right)} B_G \leq \frac{\lambda B_G}{\sqrt{s}} \label{eq:term_5}. 
    \end{equation}

    Combining the inequalities \eqref{eq:xi_bound} and \eqref{eq:term_2}-\eqref{eq:term_5}, we have the following upper bound for instantaneous regret $r_t$:
    \begin{multline}
        r_t \leq \tilde{\Delta}+ \sqrt{4L\Delta\log\left(n/\delta\right)}\\ + \sum_{i \in [h]} \sum_{a \in \left\{\tilde{a}_t^*,a_t\right\}}  \sqrt{2sN_a B_z^2\log\left(N_a/\delta\right)}\left\Vert \Xi_t\left(\mathbf{c}_i\right)\right\Vert_{V_a^t\left(\mathbf{c}_i\right)^{-1}}\\+ 
        \sqrt{2B_{\varepsilon}^2\log\left(\frac{1}{\delta}\left(\frac{s\lambda + N_a B_{\Xi}}{s\lambda}\right)^{s/2}\right)}\left\Vert \Xi_t\left(\mathbf{c}_i\right)\right\Vert_{V_a^t\left(\mathbf{c}_i\right)^{-1}} \\ + 2B_G B_{\Xi}\sqrt{sN_a}\left\Vert \Xi_t\left(\mathbf{c}_i\right)\right\Vert_{V_a^t\left(\mathbf{c}_i\right)^{-1}}
        + \frac{\lambda B_G}{\sqrt{s}}\left\Vert \Xi_t\left(\mathbf{c}_i\right)\right\Vert_{V_a^t\left(\mathbf{c}_i\right)^{-1}} \\
        + \phi_{\tilde{A}_t^*}^t - \phi_{A_t}^t + \beta_{\tilde{A}_t^*}^t - \beta_{A_t}^t  + \gamma_{\tilde{A}_t^*}^t - \gamma_{A_t}^t  \nonumber. 
    \end{multline}

    Regret is the summation of instantaneous regrets, i.e. $R_n = \sum_{t=1}^n r_t$. Therefore, letting $N_a \rightarrow n$ for each term, we have the upper-bound for regret $R_n$
    \begin{multline}
        R_n \leq n \tilde{\Delta} + n \sqrt{4L\Delta\log\left(n/\delta\right)} \\ + h d^s\sum_{a \in [d]} \sum_{t \in \mathcal{T}_{a\mid \mathbf{c}_i}} 
        \overline{B}_{\sigma} \left\Vert \Xi_t\left(\mathbf{c}_i\right)\right\Vert_{V_a^t\left(\mathbf{c}_i\right)^{-1}} \\ + h d^s \max_{\tilde{a} \in [d]} \sum_{t \in \mathcal{T}_{\tilde{a} \mid \mathbf{c}_i}} 
        \overline{B}_{\sigma} \left\Vert \Xi_t\left(\mathbf{c}_i\right)\right\Vert_{V_{\tilde{a} }^t\left(\mathbf{c}_i\right)^{-1}} \\ + \sum_{t \in [n]}\phi_{\tilde{A}_t^*}^t - \phi_{A_t}^t + \beta_{\tilde{A}_t^*}^t - \beta_{A_t}^t  + \gamma_{\tilde{A}_t^*}^t - \gamma_{A_t}^t\label{eq:regret_inequality_3}. 
    \end{multline}
    where $\overline{B}_{\sigma}$ is defined as \eqref{eq:B_sigma_def}. The values $hd^s$ are multiplied by the model error terms since there are $d^s$ models $G_{a}\left(\mathbf{c}_i\right)$ for each $i \in [h]$. Using the Cauchy-Schwarz inequality in \eqref{eq:regret_inequality_3}, we can upper bound the regret $R_n$ as follows:
    \begin{multline}
        R_n \leq n \tilde{\Delta} + n\sqrt{4L\Delta\log\left(n/\delta\right)} \\ 
        + h d^s\sum_{a \in [d]} \overline{B}_{\sigma} \sqrt{N_a} \sqrt{\sum_{t \in \mathcal{T}_{a\mid \mathbf{c}_i}} \left\Vert \Xi_t\left(\mathbf{c}_i\right)\right\Vert_{V_a^t\left(\mathbf{c}_i\right)^{-1}}^2} \\ + h d^s \max_{\tilde{a} \in [d]} \sum_{t \in \mathcal{T}_{\tilde{a} \mid \mathbf{c}_i}} 
        \overline{B}_{\sigma} \left\Vert \Xi_t\left(\mathbf{c}_i\right)\right\Vert_{V_{\tilde{a} }^t\left(\mathbf{c}_i\right)^{-1}} \\ + \sum_{t \in [n]}\phi_{\tilde{A}_t^*}^t - \phi_{A_t}^t + \beta_{\tilde{A}_t^*}^t - \beta_{A_t}^t  + \gamma_{\tilde{A}_t^*}^t - \gamma_{A_t}^t\label{eq:regret_inequality_4}. 
    \end{multline}

    Next, we now have to bound the rest of the terms in inequality \eqref{eq:regret_inequality_4}. This is broken down into the following steps:
    
    \textbf{Bounding} $\sum_{t \in \mathcal{T}_{a\mid \mathbf{c}_i}} \left\Vert \Xi_t\left(\mathbf{c}_i\right)\right\Vert_{V_a^t\left(\mathbf{c}_i\right)^{-1}}^2$: Since $V_a^t\left(\mathbf{c}_i\right) \succeq \lambda I_s$, then the following inequalities are true
    \begin{align}
        & \left\Vert \Xi_t\left(\mathbf{c}_i\right)\right\Vert_{V_a^t\left(\mathbf{c}_i\right)^{-1}}^2 \nonumber \\
        & ~~~~\leq \min\left\{\frac{\left\Vert \Xi_t\left(\mathbf{c}_i\right)\right\Vert_2}{\sqrt{\lambda}},\left\Vert \Xi_t\left(\mathbf{c}_i\right)\right\Vert_{V_a^t\left(\mathbf{c}_i\right)^{-1}}^2\right\} \nonumber \\
        & ~~~~\leq \max\left\{1,\frac{\left\Vert \Xi_t\left(\mathbf{c}_i\right)\right\Vert_2}{\sqrt{\lambda}}\right\} \min\left\{1,\left\Vert \Xi_t\left(\mathbf{c}_i\right)\right\Vert_{V_a^t\left(\mathbf{c}_i\right)^{-1}}^2\right\} \nonumber \\
        & ~~~~\leq \max\left\{1,\frac{B_{\Xi}}{\sqrt{\lambda}}\right\}\min\left\{1,\left\Vert \Xi_t\left(\mathbf{c}_i\right)\right\Vert_{V_a^t\left(\mathbf{c}_i\right)^{-1}}^2\right\} \nonumber,
    \end{align}
    \begin{multline}
        \Rightarrow \sum_{t \in \mathcal{T}_{a\mid \mathbf{c}_i}} \left\Vert \Xi_t\left(\mathbf{c}_i\right)\right\Vert_{V_a^t\left(\mathbf{c}_i\right)^{-1}}^2 \\ \leq \max\left\{1,\frac{B_{\Xi}}{\sqrt{\lambda}}\right\}\sum_{t \in \mathcal{T}_{a\mid \mathbf{c}_i}}  \min\left\{1,\left\Vert \Xi_t\left(\mathbf{c}_i\right)\right\Vert_{V_a^t\left(\mathbf{c}_i\right)^{-1}}^2\right\} \label{eq:term_8}. 
    \end{multline}
    
    According to Lemma 19.4 (the eclipse potential lemma) in \cite{lattimore2020bandit}, the bound on the right side of inequality \eqref{eq:term_8} is 
    \begin{equation}
        \sum_{t \in \mathcal{T}_{a\mid \mathbf{c}_i}}  \min\left\{1,\left\Vert \Xi_t\left(\mathbf{c}_i\right)\right\Vert_{V_a^t\left(\mathbf{c}_i\right)^{-1}}^2\right\} \leq 2s\log\left(\frac{s\lambda + N_a B_{\Xi}}{s\lambda}\right)\nonumber, 
    \end{equation}
    \begin{multline}
        \Rightarrow \sqrt{\sum_{t \in \mathcal{T}_{a\mid \mathbf{c}_i}} \left\Vert \Xi_t\left(\mathbf{c}_i\right)\right\Vert_{V_a^t\left(\mathbf{c}_i\right)^{-1}}^2}\\ \leq \max\left\{1,\frac{B_{\Xi}}{\sqrt{\lambda}}\right\}\sqrt{2s\log\left(\frac{s\lambda + N_a B_{\Xi}}{s\lambda}\right)} \label{eq:term_6}. 
    \end{multline}

    \textbf{Bounding} $\sum_{t \in \mathcal{T}_{\tilde{a} \mid \mathbf{c}_i}} \left\Vert \Xi_t\left(\mathbf{c}_i\right)\right\Vert_{V_{\tilde{a} }^t\left(\mathbf{c}_i\right)^{-1}}$: Since $V_a^t\left(\mathbf{c}_i\right) \succeq \lambda I_s$, then 
    \begin{equation}
        \sum_{t \in \mathcal{T}_{\tilde{a} \mid \mathbf{c}_i}} 
        \left\Vert \Xi_t\left(\mathbf{c}_i\right)\right\Vert_{V_{\tilde{a} }^t\left(\mathbf{c}_i\right)^{-1}} \leq \sum_{t \in \mathcal{T}_{\tilde{a}\mid \mathbf{c}_i}} \frac{\left\Vert \Xi_t\left(\mathbf{c}_i\right)\right\Vert_2}{\sqrt{\lambda}
        } \nonumber. 
    \end{equation}
    \begin{equation}
        \Rightarrow \sum_{t \in \mathcal{T}_{\tilde{a} \mid \mathbf{c}_i}} 
        \left\Vert \Xi_t\left(\mathbf{c}_i\right)\right\Vert_{V_{\tilde{a} }^t\left(\mathbf{c}_i\right)^{-1}} \leq N_a \frac{B_{\Xi}}{\sqrt{\lambda}
        } \nonumber. 
    \end{equation}
    
    \textbf{Bounding} $\phi_{\tilde{A}_t^*}^t - \phi_{A_t}^t$: Recall that $\phi_{A_t}^t$ is defined as \eqref{eq:phi_def}. In Theorem \ref{theorem:model_error} the bound for $\phi_{A_t}^t$ is 
    \begin{equation}
        \left\vert \phi_{A_t}^t\right\vert \leq 2B_G B_{\Xi} \nonumber, 
    \end{equation}
    \begin{equation}
        \Rightarrow \sum_{t\in [n]} \phi_{\tilde{A}_t^*}^t - \phi_{A_t}^t \leq 4nB_G B_{\Xi}\nonumber,
    \end{equation}
    which is satisfied with a probability of at least $1-2\delta$. 
    
    \textbf{Bounding} $\beta_{\tilde{A}_t^*}^t - \beta_{A_t}^t$: Recall that $\beta_{A_t}^t$ is defined as \eqref{eq:beta_def}. In Theorem \ref{theorem:model_error} the bound for $\beta_{A_t}^t$ is 
    \begin{equation}
        \left\vert \beta_{A_t}^t\right\vert\leq \sqrt{2nB_z^2 \log\left(n/\delta\right)} \nonumber, 
    \end{equation}
    \begin{equation}
        \Rightarrow \sum_{t\in [n]} \beta_{\tilde{A}_t^*}^t - \beta_{A_t}^t \leq 2n\sqrt{2nB_z^2\log\left(n/\delta\right)} \nonumber,
    \end{equation}
    which is satisfied with a probability of at least $1-2\delta$. 
    
    \textbf{Bounding} $\sum_{t\in \mathcal{T}_{a\mid \mathbf{c}_i}} \gamma_{\tilde{A}_t^*}^t - \gamma_{A_t}^t$: Since $\sum_{t\in \mathcal{T}_{a\mid \mathbf{c}_i}} \gamma_{\tilde{A}_t^*}^t - \gamma_{A_t}^t$ is $n B_{\varepsilon}^2$ Gaussian distributed, then the summation of has the following bound which is satisfied with a probability of at least $1-\delta$:
    \begin{equation}
        \sum_{t\in [n]} \gamma_{\tilde{A}_t^*}^t - \gamma_{A_t}^t \leq \sqrt{2nB_{\varepsilon}^2\log\left(1/\delta\right)} \label{eq:term_7}. 
    \end{equation}

    Combining the terms \eqref{eq:term_5}-\eqref{eq:term_7} for \eqref{eq:regret_inequality_4}, we have the following regret upper bound
    \begin{multline}
        R_n \leq n \tilde{\Delta} + n \sqrt{4L\Delta\log\left(n/\delta\right)} \\ + hd^s \sum_{a \in [d]} \overline{B}_{\sigma} \max\left\{1,\frac{B_{\Xi}}{\sqrt{\lambda}}\right\} \sqrt{N_a} \sqrt{2s\log\left(\frac{s\lambda + N_a B_{\Xi}}{s\lambda}\right)} \\
        h d^s \overline{B}_{\sigma} N_a \frac{B_{\Xi}}{\sqrt{\lambda}}
        + 2n \sqrt{2B_z^2\log\left(n/\delta\right)} \\ + 4n B_G B_{\Xi} + \sqrt{2nB_{\varepsilon}^2\log\left(1/\delta\right)} \nonumber. 
    \end{multline}

    Letting $N_a \rightarrow n$ we now arrive at the following regret upper bound \eqref{eq:regret_upper_bound}.
    
\end{proof}

\bibliographystyle{IEEEtran}
\bibliography{IEEEabrv,autosam}{}

\end{document}